\pdfoutput=1  % PDFLaTeX https://arxiv.org/help/submit_tex#pdflatex
\documentclass{article} % For LaTeX2e

\usepackage[preprint]{neurips_2021}

\usepackage[utf8]{inputenc} % allow utf-8 input
\usepackage[T1]{fontenc}    % use 8-bit T1 fonts
\usepackage{hyperref}       % hyperlinks
\usepackage{url}            % simple URL typesetting
\usepackage{booktabs}       % professional-quality tables
\usepackage{amsfonts}       % blackboard math symbols
\usepackage{nicefrac}       % compact symbols for 1/2, etc.
\usepackage{microtype}      % microtypography
\usepackage{xcolor}         % colors

\usepackage{amsmath}
\usepackage{amsthm}
\usepackage{thmtools}
\usepackage{caption}
\usepackage{float}

\setcitestyle{numbers,square}

\declaretheorem[name=Guiding example,numbered=no]{example}

\DeclareMathOperator*{\argmax}{arg\,max}

\newcommand{\indep}{\perp \!\!\! \perp}

\usepackage{tikz}
\usetikzlibrary{positioning}

\tikzset{
node/.style={draw, circle, thick, minimum size=25pt, align=center, inner sep=1pt},
edge.u/.style={-, >=latex, thick, shorten >=1pt, shorten <=1pt},
edge.d/.style={->, >=latex, thick, shorten >=1pt, shorten <=1pt},
edge.b/.style={<->, >=latex, thick, shorten >=1pt, shorten <=1pt}
}

\usepackage{tabularx}
\usepackage{multirow}
\usepackage{siunitx}

\graphicspath{{figures/}}

\title{Causal Reinforcement Learning\\
using Observational and Interventional Data}

\author{Maxime Gasse\\
Polytechnique Montréal\\
Montréal QC, Canada\\
\texttt{maxime.gasse@polymtl.ca} \\
\And
Damien Grasset \\
IRT Saint Exupéry Canada \\
Montréal QC, Canada\\
\texttt{damien.grasset@irt-saintexupery.com} \\
\And
Guillaume Gaudron \\
Ubisoft La Forge \\
Bordeaux, France \\
\texttt{guillaume.gaudron@ubisoft.com} \\
\And
Pierre-Yves Oudeyer\\
Inria Bordeaux Sud-Ouest \\
Bordeaux, France \\
\texttt{pierre-yves.oudeyer@inria.fr}
}

\begin{document}

\maketitle

\begin{abstract}
Learning efficiently a causal model of the environment is a key challenge of model-based RL agents operating in POMDPs. We consider here a scenario where the learning agent has the ability to collect online experiences through direct interactions with the environment (interventional data), but has also access to a large collection of offline experiences, obtained by observing another agent interacting with the environment (observational data). A key ingredient, that makes this situation non-trivial, is that we allow the observed agent to interact with the environment based on hidden information, which is not observed by the learning agent. We then ask the following questions: can the online and offline experiences be safely combined for learning a causal model ? And can we expect the offline experiences to improve the agent's performances ? To answer these questions, we import ideas from the well-established causal framework of do-calculus, and we express model-based reinforcement learning as a causal inference problem. Then, we propose a general yet simple methodology for leveraging offline data during learning. In a nutshell, the method relies on learning a latent-based causal transition model that explains both the interventional and observational regimes, and then using the recovered latent variable to infer the standard POMDP transition model via deconfounding. We prove our method is correct and efficient in the sense that it attains better generalization guarantees due to the offline data (in the asymptotic case), and we illustrate its effectiveness empirically on synthetic toy problems. Our contribution aims at bridging the gap between the fields of reinforcement learning and causality.
\end{abstract}

\section{Introduction}
\label{sec:intro}

As human beings, a key ingredient in our learning process is experimentation: we perform actions in our environment and we measure their outcomes. Another ingredient, maybe less understood, is observation: we observe the behaviour of other people, animals, or even plants interacting and evolving in our environment. A whole field of science, astronomy, relies on the observation of celestial bodies in the sky, on which experimentation is virtually impossible. And yet it is well-known that observation alone is not sufficient to infer how our environment works, or more precisely to predict the outcome of our own actions\footnote{Simply put, correlation does not imply causation. Or, citing \citet{journal/statsur/Pearl09}, ``behind every causal conclusion there must lie some causal assumption that is not testable in observational studies''.}, especially when the behaviours we observe depend on hidden information. So which role exactly does observation play during learning ? In particular, how do we combine observation and experimentation ?

In the context of reinforcement learning (RL), a related question is the following: can offline data, resulting from observations, be combined with online data resulting from experimentation, in order to improve the performance of a learning agent ? In the Markov Decision Process (MDP) setting, where the agent observes the entire state of the environment, the answer is straightforward and practical solutions exist, leading to the fastly growing field of offline reinforcement learning \citep{journal/alo/LangeGR12,levine2020offline} where large databases of demonstrations can be efficiently leveraged.
In the more general Partially-Observable MDP (POMDP) setting however, the question turns out to be much more challenging. A typical example is in the context of medicine, where offline data is collected from physicians which may rely on information absent from their patient's medical records, such as their wealthiness or their lifestyle. Suppose that wealthy patients in general get prescribed specific treatments by their physicians, because they can afford it, while being less at risk to develop severe conditions regardless of their treatment, because they can also afford a healthier lifestyle. This creates a spurious correlation called confounding, and will cause a naive recommender system to wrongly infer that a treatment has positive health effects. A second example is in the context of autonomous driving, where offline data is collected from human drivers who have a wider field of vision than the camera on which the robot driver relies. Suppose human drivers push the brakes when they see a person waiting to cross the street, and only when the person walks in front of the car it enters the camera's field of vision. Then, again, a naive robot might wrongly infer from its observations that whenever brakes are pushed, a person appears in front of the car. Suppose now that the robot's objective is to never collide with someone, it might deduce that never pulling the brakes is a good strategy. Of course, in both those situations, the learning agent will eventually infer the right causal effects of its actions if it collects enough online data from its own interactions. However, in both those situations also, performing many interventions for the sole purpose of seeing what happens is not really affordable, while collecting offline data by observing the behaviour of human agents is much more realistic.

In this paper we study the question of combining offline and online data under the Partially-Observable Markov Decision Process (POMDP) setting, by importing tools and ideas from the well-established field of causality \citep{book/Pearl09} into the model-based RL framework. Our contribution is three-fold:
\begin{enumerate}
    \item We formalise model-based RL as a causal inference problem using the framework of $do$-calculus
    \citep{conf/uai/Pearl12}, which allows us to reason formally about online and offline scenarios in a natural manner (Section~\ref{sec:causal-rl}).
    \item We present a generic method for combining offline and online data in model-based RL (Section~\ref{sec:contribution}), with a formal proof of correctness even when the offline policy relies on privileged hidden information (confounding variable), and a proof of efficiency in the asymptotic case (with respect to using online data only).
    \item We propose a practical implementation of our method, and illustrate its effectiveness in two experiments with synthetic toy problems (Section~\ref{sec:experiments}).
\end{enumerate}

While our proposed method can be formulated outside of the $do$-calculus framework, in this paper we hope to demonstrate that $do$-calculus offers a principled and intuitive tool to reason about model-based RL. By relating common concepts from RL and causality, we wish that our contribution will ultimately help to bridge the gap between the two communities.

\section{Background}
\label{sec:background}

\subsection{Notation}

In this paper, upper-case letters in italics denote random variables (e.g. $X,Y$), while their lower-case counterpart denote their value (e.g. $x,y$) and their calligraphic counterpart their domain (e.g., $x \in \mathcal{X}$). We consider only discrete random variables. To keep our notation uncluttered, with a slight abuse of notations and use $p(x)$ to denote sometimes the event probability $p(X=x)$, and sometimes the whole probability distribution of $X$, which should be clear from the context. In the context of sequential models we also distinguish random variables with a temporal index $t$, which might be fixed (e.g., $o_0,o_1$ ), or undefined (e.g., $p(s_{t+1}|s_t,a_t)$ denotes at the same time the distributions $p(s_1|s_0,a_0)$ and $p(s_2|s_1,a_1)$). We also adopt a compact notation for sequences of contiguous variables (e.g., $s_{0 \to T} = (s_0,\dots,s_T) \in \mathcal{S}^{T+1}$ ), and for summation over sets ($\sum_{x \in \mathcal{X}} \iff \sum_x^\mathcal{X}$).
We assume the reader is familiar with the concepts of conditional independence ($X \indep Y \mid Z$) and probabilistic graphical models based on directed acyclic graphs (DAGs), which can be found in most introductory textbooks, e.g. \citet{books/Pearl89,book/studeny05,books/KollerF09}.

\subsection{Do-calculus}

Several frameworks exist in the literature for reasoning about causality \citep{book/Pearl09, book/ImbensRubin15}.
Here we follow the framework of Judea Pearl, whose concept of \emph{ladder of causation} is particularly relevant to answer RL questions. The first level of the ladder, \emph{association}, relates to the observation of an external agent acting in the environment, while the second level, \emph{intervention}, relates the question of what will happen to the environment as a result of one's own actions. The tool of do-calculus \citep{conf/uai/Pearl12} acts as a bridge between these two levels, and relates interventional distributions, such as $p(y|do(x))$, to observational distributions, such as $p(y|x)$, in causal systems that can be expressed as DAGs.
In a nutshell, do-calculus allows for measuring changes in the distribution of random variables $\{X,Y,Z,\dots\}$, when one performs an arbitrary intervention $do(x)$ which forces some variables to take values $X=x$ regardless of their causal ancestors. It relies on a complete set of rules \citep{conf/uai/HuangV06,conf/aaai/ShpitserP06}, which allow for the following equivalences when specific structural conditions are met in the causal DAG:
\begin{itemize}
    \item R1: insertion/deletion of observations $p(y|do(x),z,w) = p(y|do(x),w)$,
    \item R2: action/observation exchange $p(y|do(x),do(z),w) = p(y|do(x),z,w)$,
    \item R3: insertion/deletion of actions $p(y|do(x),do(z),w) = p(y|do(x),w)$.
\end{itemize}
We refer the reader to \citet{conf/uai/Pearl12} for a thorough introduction to $do$-calculus. In this paper, we will use these rules to derive formal solutions to model-based RL in various POMDP settings.

\subsection{Partially-Observable Markov Decision Process}

We consider Partially-Observable Markov Decision Processes (POMDPs) of the form $M = (\mathcal{S}, \mathcal{O}, \mathcal{A}, p_\textit{init}, p_\textit{obs}, p_\textit{trans}, r)$, with hidden states $s \in \mathcal{S}$, observations $o \in \mathcal{O}$, actions $a \in \mathcal{A}$, initial state distribution
$p_\textit{init}(s_0)$, state transition distribution
$p_\textit{trans}(s_{t+1}|s_t,a_t)$, observation distribution 
$p_\textit{obs}(o_t | s_t)$, and reward\footnote{Without loss of generality we consider the reward to be part of the observation $o_t$ to simplify our notation.} function $r: \mathcal{O} \to \mathbb{R}$. For simplicity we assume episodic tasks with finite horizon $H$. We further denote a complete trajectory $\tau=(o_0,a_0,\dots,o_H)$, and for convenience we introduce the concept of a history at time $t$, $h_t=(o_0,a_0,\dots,o_t)$.

A common control scenario for POMDPs is when actions are decided based on all the available information from the past. We call this the \emph{standard POMDP setting}. The control mechanism can be represented as a stochastic policy $\pi(a_t|h_t)$, which together with the POMDP dynamics $p_\textit{init}$, $p_\textit{obs}$ and $p_\textit{trans}$ defines a probability distribution over trajectories $\tau$,
\begin{equation*}
    p_\textit{std}(\tau) = \sum_{s_{0 \to |\tau|}}^{\mathcal{S}^{|\tau|+1}} p_\textit{init}(s_0) p_\textit{obs}(o_0|s_0) \prod_{t=0}^{|\tau|-1} \pi(a_t|h_t) p_\textit{trans}(s_{t+1}|s_t,a_t)p_\textit{obs}(o_{t+1}|s_{t+1})
    \text{.}
\end{equation*}
This whole data-generation mechanism can be represented visually as a DAG, represented in Figure~\ref{fig:pomdp-unc}. A key characteristic in this setting is that $A_t \indep S_t \mid H_t$ is always true, that is, every action is independent of the current state given the history.

\subsection{Model-based RL}

Assuming the objective is the long-term reward, the POMDP control problem formulates as:
\begin{equation}
    \label{eq:pomdp-control}
    \pi^\star = \argmax_\pi \mathop{\mathbb{E}}_{\tau \sim p_\textit{std}}\left[\sum_{t=0}^{|\tau|}r(o_t)\right]
    \text{.}
\end{equation}%
Model-based RL relies on the estimation of the POMDP transition model $p_\textit{std}(o_{t+1}|h_t,a_t)$ to solve (\ref{eq:pomdp-control}), which decomposes into two sub-problems:
\begin{enumerate}
    \item learning: given a dataset $\mathcal{D}$, estimate a transition model $\hat{q}(o_{t+1}|h_t,a_t) \approx p_\textit{std}(o_{t+1}|h_t,a_t)$;
    \item planning: given a history $h_t$ and a transition model $\hat{q}$, decide on an optimal action $a_t$.
\end{enumerate}%
As we will see shortly, the transition model $\hat{q}$ seeked by model-based RL is inherently causal \citep{chapter/Gershman17}. In this work we consider only the first problem above, that is, learning the (causal) POMDP transition model from data.

\section{Model-based RL as a causal inference}
\label{sec:causal-rl}

Decision problems, such as those arising in POMDPs, can naturally be formulated in terms of causal queries where actions directly translate into $do$ statements. For example, given past information about the POMDP process, what will be the causal effect of an action (intervention) on future rewards ?

\begin{example}
    Consider a door, a light, and two buttons A and B. The light is red 60\% of the time, and green the rest of the time. When the light is red, button A opens the door, while when the light is green, then button B opens the door. I am told that the mechanism responsible for opening the door depends on both the light color and the button pressed ($\textit{light} \rightarrow \textit{door} \leftarrow \textit{button}$), but I am not given the mechanism itself. Suppose now that I am colorblind, and I want to open the door. Which button should I press ? In the do-calculus framework, the question I am asking is
    \begin{equation*}
        \argmax_{\textit{button}\in \{A,B\}} p(\textit{door=open}|do(\textit{button}))
        \text{.}
    \end{equation*}
\end{example}

\begin{figure}[tbp]
\centering
\hfill
\begin{minipage}{0.45\textwidth}
    \centering
    \begin{tikzpicture}[scale=1.35, every node/.style={scale=0.8}]
        \node[node,label={$p_\textit{init}$}] (S0) at (-2, 0) {$S_0$};
        \node[node,label={$p_\textit{trans}$}] (St) at (0, 0) {$S_t$};
        \node[node,label={$p_\textit{trans}$}] (Stp1) at (2, 0) {$S_{t+1}$};
        \node[node,label={80:$p_\textit{obs}$},fill=gray!25] (O0) at (-2, -1) {$O_0$};
        \node[node,label={80:$p_\textit{obs}$},fill=gray!25] (Ot) at (0, -1) {$O_t$};
        \node[node,label={80:$p_\textit{obs}$}] (Otp1) at (2, -1) {$O_{t+1}$};
        \node[node,label={270:$\pi$}] (At) at (1, -1) {$A_t$};
        \node[node,label={270:$\pi$},fill=gray!25] (Atm1) at (-1, -1) {$A_{t-1}$};
        \draw[edge.d] (St) to (Stp1);
        \draw[edge.d] (St) to (Ot);
        \draw[edge.d] (Stp1) to (Otp1);
        \draw[edge.d] (S0) to (O0);
        \draw[edge.d, dashed] (O0) to (Atm1);
        \draw[edge.d, dashed] (S0) to (St);
        \draw[edge.d] (Atm1) to (St);
        \draw[edge.d] (Ot) to (At);
        \draw[edge.d] (At) to (Stp1);
        \draw[edge.d] (Atm1) to[out=-35, in=-155] (At);
        \draw[edge.d] (O0) to[out=-35, in=-130] (At);
    \end{tikzpicture}
    \vspace{-0.6cm}
    \caption{Standard POMDP setting.}
    \label{fig:pomdp-unc}
\end{minipage}
\hfill
\begin{minipage}{0.45\textwidth}
    \centering
    \begin{tikzpicture}[scale=1.35, every node/.style={scale=0.8}]
        \node[node,label={$p_\textit{init}$}] (S0) at (-2, 0) {$S_0$};
        \node[node,label={$p_\textit{trans}$}] (St) at (0, 0) {$S_t$};
        \node[node,label={$p_\textit{trans}$}] (Stp1) at (2, 0) {$S_{t+1}$};
        \node[node,label={80:$p_\textit{obs}$},fill=gray!25] (O0) at (-2, -1) {$O_0$};
        \node[node,label={80:$p_\textit{obs}$},fill=gray!25] (Ot) at (0, -1) {$O_t$};
        \node[node,label={80:$p_\textit{obs}$}] (Otp1) at (2, -1) {$O_{t+1}$};
        \node[node,label={270:$\pi_{prv}$}] (At) at (1, -1) {$A_t$};
        \node[node,label={270:$\pi_{prv}$},fill=gray!25] (Atm1) at (-1, -1) {$A_{t-1}$};
        \draw[edge.d] (St) to (Stp1);
        \draw[edge.d] (St) to (Ot);
        \draw[edge.d, red] (St) to (At);
        \draw[edge.d] (Stp1) to (Otp1);
        \draw[edge.d] (S0) to (O0);
        \draw[edge.d, dashed] (S0) to (St);
        \draw[edge.d, dashed, red] (S0) to (Atm1);
        \draw[edge.d, dashed] (O0) to (Atm1);
        \draw[edge.d] (Atm1) to (St);
        \draw[edge.d] (Ot) to (At);
        \draw[edge.d] (At) to (Stp1);
        \draw[edge.d] (Atm1) to[out=-35, in=-155] (At);
        \draw[edge.d] (O0) to[out=-35, in=-130] (At);
    \end{tikzpicture}
    \vspace{-0.6cm}
    \caption{Privileged POMDP setting.}
    \label{fig:pomdp-con}
\end{minipage}
\hfill
\end{figure}

\subsection{The interventional regime}

In the interventional regime, we assume a dataset $\mathcal{D}_\textit{int}$ of episodes $\tau$ collected in the standard POMDP setting from an arbitrary decision policy $\pi(a_t|h_t)$,
$$
    \mathcal{D}_\textit{int} \sim p_\textit{init}, p_\textit{trans}, p_\textit{obs}, \pi
    \text{.}
$$

Let us now adopt a causal perspective and reason in terms of interventions in the causal system, depicted in Figure~\ref{fig:pomdp-unc}. Consider that we want to control the system, that is, replace $\pi$ with $\pi^\star$, in order to maximize a long-term outcome. Then, evaluating the effect of each action on the system is a causal inference problem. In order to decide on the best first action $a_0$ given $h_0=(o_0)$, one must evaluate a series of causal queries in the form $p_\textit{std}(o_{1}|o_0,do(a_0))$, then $p_\textit{std}(o_{2}|o_0,do(a_0),o_1,do(a_1))$, and so on, and finally using those causal distributions for planning, by solving a Bellman equation. Conveniently, in the interventional regime, applying rule R2 of do-calculus on the causal DAG results in those queries being trivially identifiable from $p_\textit{std}(\tau)$. In fact, those queries exactly boil down to the standard POMDP transition model that model-based RL seeks to estimate,
\begin{equation}
    \label{eq:pomdp-standard-transition}
    p_\textit{std}(o_{t+1}|o_{0 \to t},do(a_{0 \to t})) = p_\textit{std}(o_{t+1}|h_{t},a_{t})
    \text{.}
\end{equation}
As such, model-based RL can be naturally reinterpreted in terms of causal inference. Also, a convenient property in this regime is that $p_\textit{std}(o_{t+1}|h_t,a_t)$ does not depend on the control policy $\pi$ that was used to build the dataset $\mathcal{D}_\textit{int}$. The only requirement, in order to estimate transition probabilities for every $h_t,a_t$ combination, is that $\pi$ has a non-zero chance to explore every action, that is, $\pi(a_t|h_t)>0$, $\forall a_t,h_t$. Then, an unbiased estimate of the standard POMDP transition model can be obtained simply via log-likelihood maximization:
\begin{equation}
    \label{eq:direct-learning}
    \hat{q}= \argmax_{q \in \mathcal{Q}} \sum_{\tau}^{\mathcal{D}_\textit{int}}\sum_{t=0}^{|\tau|-1} \log q(o_{t+1}|h_t,a_t)
\text{.}
\end{equation}

In some situations it is very reasonable to assume an interventional regime, for example when it is known to hold by construction. This is the case with online RL data, as the learning agent itself explicitly controls the data-collection policy $\pi(a_t|h_t)$. But it can also be the case with offline RL data, if one knows that the data-collection policy did not use any additional information besides the information available to the learning agent, $h_t$. In Atari video games for example, it is hard to imagine a human player using any kind of privileged information related to the machine's internal state $s_t$ other than the video and audio outputs from the game.

\begin{example}
     Consider again our door example. If I am able to observe myself or another colorblind person interacting with the door, then I know that which button is pressed is unrelated to which color the light is ($\textit{light} \not\rightarrow \textit{button}$). Then I can directly estimate the causal effect of the button on the door,
     \begin{equation*}
         p(\textit{door=open}|do(\textit{button})) = p(\textit{door=open}|\textit{button})
         \text{.}
     \end{equation*}
     Whichever policy is used to collect $(\textit{button}, \textit{door})$ samples\footnote{One assumption though is strict positivity, $\pi(\textit{button})>0$ $\forall \textit{button}$, so that both buttons are pressed.}, eventually I realise that button A has more chances of opening the door (60\%) than button B (40\%), and thus is the optimal choice.
\end{example}

\subsection{The observational regime}

In the observational regime, we assume a dataset $\mathcal{D}_\textit{obs}$ of episodes $\tau$ collected in the \emph{privileged POMDP setting}, depicted in Figure~\ref{fig:pomdp-con}. In this setting episodes are collected from an external agent who has access to privileged information, in the extreme case the whole POMDP state $s_t$, which the learning agent can not observe\footnote{Note that our only assumption is that this external agent has access to privileged information. We do not assume it acts optimally with respect to the learning agent's reward, or any other reward.}. In this setting we denote the data-generating control policy $\pi_{prv}(a_t|h_t,s_t)$, such that
$$
    \mathcal{D}_\textit{obs} \sim p_\textit{init}, p_\textit{trans}, p_\textit{obs}, \pi_{prv}
    \text{.}
$$%
We denote the whole episode distribution resulting from $p_\textit{init}, p_\textit{trans}, p_\textit{obs}$ and $\pi_{prv}$ as $p_\textit{prv}(\tau)$. A key characteristic in this setting is that now $A_t \indep S_t \mid H_t$ can not be assumed to hold any more.

Let us reason here again in terms of causal inference from the causal system depicted in Figure~\ref{fig:pomdp-con}. For the purpose of controlling the POMDP in the standard setting, in the light of past information $h_t$, we want to evaluate the same series of causal queries as before, in the form $p_\textit{prv}(o_{t+1}|o_{0 \to t},do(a_{0 \to t}))$. This time however, those causal queries are not identifiable from $p_\textit{prv}(\tau)$. Evaluating them would require knowledge of the POMDP hidden states $s_t$, which act as confounding variables. For example, identifying the first query at $t=0$ requires at least the observation of $s_0$,
\begin{align*}
    p_\textit{prv}(o_1|o_0,do(a_0)) &= \sum_{s_0 \in \mathcal{S}} p_\textit{prv}(s_0|o_0,do(a_0)) p_\textit{prv}(o_1|s_0,o_0,do(a_0)) \\
    &= \sum_{s_0 \in \mathcal{S}} p_\textit{prv}(s_0|o_0) p_\textit{prv}(o_1|s_0,a_0)
\end{align*}
(R3 and R2 of do-calculus, then $O_{t+1} \indep H_t \mid S_t,A_t$).

In many offline RL situations, we believe that it is common to have access to POMDP trajectories for which $A_t \indep S_t \mid H_t$ can not be assumed, for example when demonstrations are collected from a human agent acting in the world (see~Section\ref{sec:intro} for examples). In such a situation, the observed trajectories may be confounded, and naively learning a causal transition model by applying (\ref{eq:direct-learning}) might result in a non-causal model, and in non-optimal planning. A natural question is then: what should be done in such a situation ? Are confounded trajectories useless ? Can we still use this data somehow for recovering a better, unbiased causal transition model ?

\begin{example}
Take again our door example, and assume I observe another person interacting with the door. I do not know whether that person is colorblind or not ($\textit{light} \rightarrow \textit{button}$ is possible). Then, without further knowledge, I cannot recover the causal queries $p(\textit{door=open}|do(\textit{button}))$ from the observed distribution $p(\textit{door},\textit{button})$. In the $do$-calculus framework, the queries are said \emph{non identifiable}. However, if that person was to tell me the light color they see before they press A or B, then I could recover those queries as follows,
\begin{equation*}
    p(\textit{door=open}|do(\textit{button})) = \sum_{\textit{light} \in \{red, green\}} p(\textit{light}) p(\textit{door=open}|\textit{light},\textit{button})
    \text{.}
\end{equation*}
This formula, called \emph{deconfounding}, eventually yields the correct causal transition probabilities regardless of the observed policy\footnote{The strict positivity condition here is $\pi(\textit{button}|\textit{light})>0$ $\forall \textit{button},\textit{light}$.}, given that enough $(\textit{light}, \textit{button}, \textit{door})$ samples are observed.
\end{example}

\section{Combining observational and interventional data}
\label{sec:contribution}

\subsection{Problem statement}

We consider a generic situation where two datasets of POMDP trajectories $\mathcal{D}_{\textit{int}}$ and $\mathcal{D}_{\textit{obs}}$ are available, sampled respectively in the interventional regime with policy $\pi_\textit{std}(a_t|h_t)$, and in the observational (potentially confounded) regime with policy $\pi_\textit{prv}(a_t|h_t,s_t)$. We then ask the following question: is there a sound way to use the observational data for improving the estimator of the standard POMDP transition model that would be recovered from the interventional data only ?

\subsection{The augmented POMDP}

We formulate the problem of learning the standard POMDP transition model from 
$\mathcal{D}_{\textit{int}}$ and $\mathcal{D}_{\textit{obs}}$ as that of inferring a structured latent-variable model. Since both datasets are sampled from the same POMDP ($p_{\textit{init}}$, $p_{\textit{trans}}$ and $p_{\textit{obs}}$) controlled in different ways (either $\pi_\textit{prv}$ or $\pi_\textit{std}$), the overall data generating process can be represented in the form of an augmented DAG, depicted in Figure~\ref{fig:pomdp-hybrid}. We simply introduce an auxiliary variable $I \in \{0,1\}$, which acts as a regime indicator \citep{arxiv/Dawid20} for differentiating between observational and interventional data. The augmented POMDP policy then becomes $\pi(a_t|h_t,s_t,i)$, such that
\begin{align*}
    \pi(a_t|h_t,s_t,i=0) &= \pi_\textit{prv}(a_t|h_t,s_t), \;\text{and} \\
    \pi(a_t|h_t,s_t,i=1) &= \pi_\textit{std}(a_t|h_t)
    \text{.}
\end{align*}

For simplicity, in the following we will refer to the joint distribution of this augmented POMDP as the true distribution $p$, and with a slight abuse of notation we will consider $\mathcal{D}_\textit{obs}$ and $\mathcal{D}_\textit{int}$ two datasets of augmented POMDP trajectories, sampled respectively under the observational regime $(\tau,i) \sim p(\tau,i|i=0)$, and the interventional regime $(\tau,i) \sim p(\tau,i|i=1)$.
The causal queries required to control the augmented POMDP can then be identified as
\begin{align*}
    p(o_{t+1}|o_{0 \to t},do(a_{0 \to t})) &= p(o_{t+1}|o_{0 \to t},do(a_{0 \to t}),i=1) \\
    &= p(o_{t+1}|h_t,a_t,i=1)
\end{align*}
(R1 of do-calculus, then R2 on the contextual causal DAG from Figure~\ref{fig:pomdp-unc}).

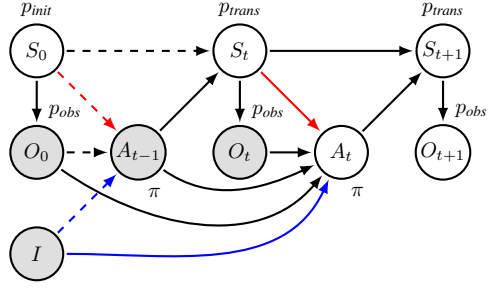
\begin{figure}[tbp]
\centering
\begin{tikzpicture}[scale=1.35, every node/.style={scale=0.8}]
    \node[node,label={$p_\textit{init}$}] (S0) at (-2, 0) {$S_0$};
    \node[node,label={$p_\textit{trans}$}] (St) at (0, 0) {$S_t$};
    \node[node,label={$p_\textit{trans}$}] (Stp1) at (2, 0) {$S_{t+1}$};
    \node[node,label={80:$p_\textit{obs}$},fill=gray!25] (O0) at (-2, -1) {$O_0$};
    \node[node,label={80:$p_\textit{obs}$},fill=gray!25] (Ot) at (0, -1) {$O_t$};
    \node[node,label={80:$p_\textit{obs}$}] (Otp1) at (2, -1) {$O_{t+1}$};
    \node[node,label={275:$\pi$}] (At) at (1, -1) {$A_t$};
    \node[node,label={275:$\pi$},fill=gray!25] (Atm1) at (-1, -1) {$A_{t-1}$};
    \node[node,fill=gray!25] (I) at (-2, -2) {$I$};
    \draw[edge.d] (St) to (Stp1);
    \draw[edge.d] (St) to (Ot);
    \draw[edge.d, red] (St) to (At);
    \draw[edge.d] (Stp1) to (Otp1);
    \draw[edge.d] (S0) to (O0);
    \draw[edge.d, dashed] (S0) to (St);
    \draw[edge.d, dashed, red] (S0) to (Atm1);
    \draw[edge.d, dashed] (O0) to (Atm1);
    \draw[edge.d] (Atm1) to (St);
    \draw[edge.d] (Ot) to (At);
    \draw[edge.d] (At) to (Stp1);
    \draw[edge.d] (Atm1) to[out=-35, in=-155] (At);
    \draw[edge.d] (O0) to[out=-35, in=-130] (At);
    \draw[edge.d, dashed, blue] (I) to (Atm1);
    \draw[edge.d, blue] (I) to[out=0, in=245] (At);
\end{tikzpicture}
\caption{augmented POMDP setting, with a policy regime indicator $I$ taking values in $\{0,1\}$ (1=interventional regime, no confounding, 0=observational regime, potential confounding), such that $\pi(a_t|h_t,s_t,i=1)=\pi(a_t|h_t,i=1)$. This additional constraint introduces a contextual independence $A_t \indep S_t \mid H_t,I=1$.}
\label{fig:pomdp-hybrid}
\end{figure}

\subsection{The augmented learning problem}

In order to learn the standard POMDP transition model $p(o_{t+1}|h_t,a_t,i=1)$ from the augmented dataset $\mathcal{D}_{\textit{obs}} \cup  \mathcal{D}_{\textit{int}}=\mathcal{D} \sim p(\tau,i)$, we propose the following two-step procedure.

\textbf{Learning}: In the first step, we fit a latent probabilistic model $\hat{q}$ to the training trajectories, constrained to respect all the independencies of our augmented POMDP. This learning problem formulates as
\begin{equation}
    \label{eq:latent-learning}
    \hat{q} = \argmax_{q \in \mathcal{Q}} \sum_{(\tau, i)}^\mathcal{D} \log q(\tau,i)
    \text{,}
\end{equation}
with $\mathcal{Q}$ the family of sequential latent probabilistic models $q$ that respect
\begin{equation*}
    \begin{gathered}
    q(\tau,i) = q(i) \sum_{z_{0 \to |\tau|}}^{\mathcal{Z}^{|\tau|+1}} q(z_0) q(o_0|z_0) \prod_{t=0}^{|\tau|-1} q(a_t|h_t,z_t,i) q(z_{t+1}|a_t,z_t) q(o_{t+1}|z_{t+1})
    \text{,} \\
    q(a_t|h_t,z_t,i=1) = q(a_t|h_t,i=1)
    \text{,}
    \end{gathered}
\end{equation*}
and $\mathcal{Z}$ the discrete latent space of the model, i.e., $z_t \in \mathcal{Z}$.

It is straightforward to see that our learning problem~(\ref{eq:latent-learning}) is that of a structured latent variable model. While the problem of learning structured latent variable models is known to be hard in general, there exists a wide range of tools and algorithms available in the literature for solving it approximately, such as the EM algorithm or the method of ELBO maximization.

\textbf{Inference}: In the second step, we recover $\hat{q}(o_{t+1}|h_t,a_t,i=1)$ as an estimator of the standard POMDP transition model. This can be done efficiently by unrolling a forward algorithm over the augmented DAG structure (see appendix for details).

Intuitively, the observational data $\mathcal{D}_\textit{obs}$ influences the interventional transition model $q(o_{t+1}|h_t,a_t,i=1)$ as follows. The learned model $q$ must fit the observational and interventional data by sharing the same building blocs $q(z_0)$, $q(o_t|z_t)$ and $q(z_{t+1}|z_t,a_t)$, and only the expert policy $q(a_t|h_t,z_t,i=0)$ offers some flexibility that allows it to differentiate between both regimes. As a result, imposing a distribution for $q(\tau|i=0)$ acts as a regularizer for $q(\tau|i=1)$.

\subsection{Theoretical guarantees}

In this section we show that our two-step approach is 1) correct, in the sense that it yields an unbiased estimator of the standard POMDP causal transition model and 2) efficient, in the sense that it yields a better estimator than the one based on interventional data only (asymptotically in the number of observational data).
First, we show that the recovered estimator is unbiased.

\begin{restatable}{proposition}{propunbiased}
\label{prop:unbiased}
Assuming $|\mathcal{Z}| \geq |\mathcal{S}|$, $\hat{q}(o_{t+1}|h_t,a_t,i=1)$ is an unbiased estimator of $p(o_{t+1}|h_t,a_t,i=1)$.
\end{restatable}
\begin{proof}
The proof is straightforward. First, we have that $\mathcal{D} \sim p(\tau,i)$. Second, we have $p \in \mathcal{Q}$, because $\mathcal{Q}$ is only restricted to the augmented POMDP constraints, and because its latent space is sufficiently large ($|\mathcal{Z}| \geq |\mathcal{S}|$). Therefore, $\hat{q}(\tau,i)$ solution of (\ref{eq:latent-learning}) is an unbiased estimator of $p(\tau,i)$, and in particular $\hat{q}(o_{t+1}|h_t,a_t,i=1)$ is an unbiased estimator of $p(o_{t+1}|h_t,a_t,i=1)$.
\end{proof}

Second, we provide bounds on $\hat{q}(o_{t+1}|h_t,a_t,i=1)$ in the asymptotic scenario $|\mathcal{D}_{\textit{obs}}| \to \infty$ (regardless of the interventional data $\mathcal{D}_{\textit{int}}$).

\begin{restatable}{theorem}{thbounds}
\label{th:bounds}
Assuming $|\mathcal{D}_{\textit{obs}}| \to \infty$, for any $\mathcal{D}_{\textit{int}}$ the recovered causal model is bounded as follows:
\begin{align*}
    \prod_{t=0}^{T-1} \hat{q}(o_{t+1}|h_t,a_t,i=1) & \geq \prod_{t=0}^{T-1} p(a_{t}|h_{t},i=0) p(o_{t+1}|h_t,a_t,i=0)
    \text{, and}\\
    \prod_{t=0}^{T-1} \hat{q}(o_{t+1}|h_{t},a_{t},i=1)
    & \leq \prod_{t=0}^{T-1} p(a_{t}|h_{t},i=0) p(o_{t+1}|h_{t},a_{t},i=0) + 1 - \prod_{t=0}^{T-1} p(a_{t}|h_{t},i=0)
    \text{,}
\end{align*}
$\forall h_{T-1},a_{T-1},T\geq 1$ where $p(h_{T-1},a_{T-1},i=0) > 0$.
\end{restatable}
\begin{proof}
See appendix.
\end{proof}

As a direct consequence, in the asymptotic case, using (infinite) observational data ensures stronger generalization guarantees for the recovered transition model than using no observational data.

\begin{restatable}{corollary}{corgeneralization}
\label{cor:generalization}
The estimator $\hat{q}(o_{t+1}|h_t,a_t,i=1)$, recovered after solving (\ref{eq:latent-learning}) with $|\mathcal{D}_\textit{obs}| \to \infty$, offers strictly better generalization guarantees than the one with $|\mathcal{D}_\textit{obs}|=0$, for any $\mathcal{D}_\textit{int}$.
\end{restatable}
\begin{proof}
There exists at least one history-action couple $(h_{T-1},a_{T-1})$, $T\geq1$, that has non-zero probability in the observational regime. This ensures that there exists a value $o_{T}$ for which $\prod_{t=0}^{T-1} p(a_t|h_{t},i=0) p(o_{t+1}|h_{t},a_t,i=0)$ is strictly positive, which in turn ensures $\hat{q}(o_{T+1}|h_T,a_T,i=1) > 0$. As a result, the family of models $\{q(o_{t+1}|h_t,a_t,i=1)\mid q \in \mathcal{Q}, q(\tau|i=0)=p(\tau|i=0)\}$ is a strict subset of the unrestricted family $\{q(o_{t+1}|h_t,a_t,i=1)\mid q \in \mathcal{Q}\}$, and thus offers strictly better generalization guarantees.
\end{proof}

\begin{example}
Let us now look at our door example in light of Theorem~\ref{th:bounds}. Assume this time that I observe many $(\textit{button},\textit{door})$ interactions from a non-colorblind person ($i=0$), who's policy is $\pi(\textit{button=A}|\textit{light=red})=0.9$ and $\pi(\textit{button=A}|\textit{light=green})=0.4$. Then I can already infer from Theorem~\ref{th:bounds} that $p(\textit{door=open}|do(\textit{button=A})) \in [0.54, 0.84]$ and $p(\textit{door=open}|do(\textit{button=B})) \in [0.24, 0.94]$. I now get a chance to interact with the door ($i=1$), and I decide to press $A$ 10 times and $B$ 10 times. I am unlucky, and my interventional study results in the following probabilities: $q(\textit{door=open}|do(\textit{button=A}))=0.5$ and $q(\textit{door=open}|do(\textit{button=B}))=0.5$. This does not coincide with my (reliable) observational study, and therefore I adjust $q(\textit{door=open}|do(\textit{button=A}))$ to its lower bound $0.54$. I now believe that pressing $A$ is more likely to be my optimal strategy.
\end{example}

\section{Experiments}
\label{sec:experiments}

We perform experiments on two synthetic toy problems, our \emph{door} bandit problem described earlier (Sections~\ref{sec:causal-rl}-\ref{sec:contribution}), and the \emph{tiger} problem from the literature \citep{conf/aaai/CassandraKL94} with a horizon $H=50$.
In both experiments we consider a uniform standard policy $\pi_\textit{std}$ and a good but noisy expert $\pi_\textit{prv}$ (see appendix for details).
To assess the performance of our method, we consider a large observational dataset $\mathcal{D}_\textit{obs}$ (512 samples), and an interventional dataset $\mathcal{D}_\textit{int}$ of varying size. We then compare the performance of the transition model $\hat{q}$ learned in three different settings: \emph{no obs}, where only interventional data ($\mathcal{D}=\mathcal{D}_\textit{int}$) is used for training; \emph{naive}, where observational data is naively combined with interventional data as if there was no confounding ($\mathcal{D}=\mathcal{D}_\textit{int} \cup \{(\tau,1)|(\tau,i) \in \mathcal{D}_\textit{obs})\}$); and \emph{augmented}, our method ($\mathcal{D}=\mathcal{D}_\textit{int} \cup \mathcal{D}_\textit{obs}$). The only difference between each setting is the training dataset. We learn $\hat{q}$ by minimizing (\ref{eq:latent-learning}) directly via stochastic gradient descent, and we use discrete probability tables for the building blocs of our transition model, $q(z_0)$, $q(o_t|z_t)$, $q(z_{t+1}|z_t,a_t)$, and $q(a_t|h_t,z_t,i=0)$, with a low-dimensional latent space $|\mathcal{Z}|=32$.

In Figure~\ref{fig:results} we report the Jensen-Shannon (JS) divergence between the recovered $\hat{q}(o_{t+1}|h_t,a_t,i=1)$ and the true transition model $p(o_{t+1}|h_t,do(a_t))$, and also the expected reward obtained when the model is used for planning. In the \emph{door} problem planning is trivial, and we compute both the JS and expected reward exactly. In the \emph{tiger} problem we use the recovered model to train a ``dreamer'' RL agent on imaginary samples $\tau \sim \hat{q}(\tau|i=1)$ from the model, using the belief states $\hat{q}(s_t|h_t)$ as features, and we compute both the JS and expected reward using a stochastic approximation (details in the appendix). As can be seen, in both experiments our method (\emph{augmented}) nicely leverages the observational data and converges faster than when no observational data is used (\emph{no obs}), or when it is used in a way that disregards a potential confounding issue (\emph{naive}). We also perform additional experiments in situations where the observed agent is perfectly good, perfectly bad, random, or strongly biased towards or against the optimal action (appendix, Section~\ref{sec:additional_results}). In all cases our method successfully leverages the observational data, and outperforms or matches the performance of the other methods. Code for reproducing our experiments is made available online\footnote{\url{https://github.com/causal-rl-anonymous/causal-rl}}.

\begin{figure}
    \centering
    \includegraphics[height=2.8cm]{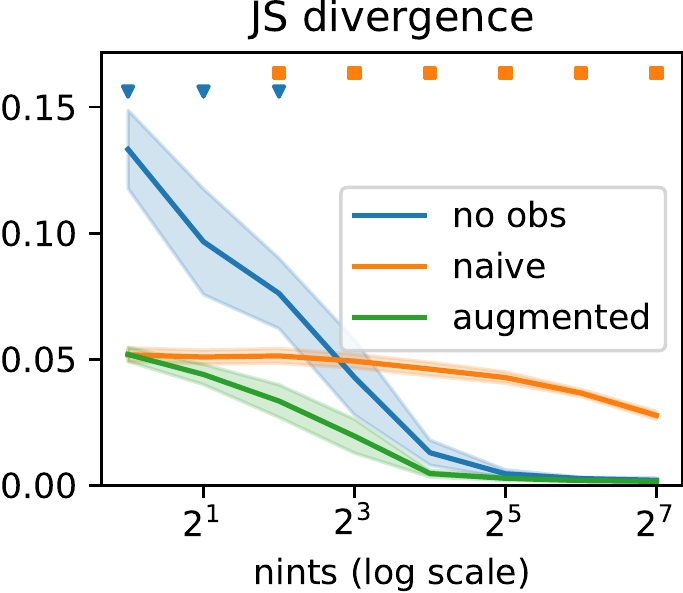}
    \hspace*{\fill}
    \includegraphics[height=2.8cm]{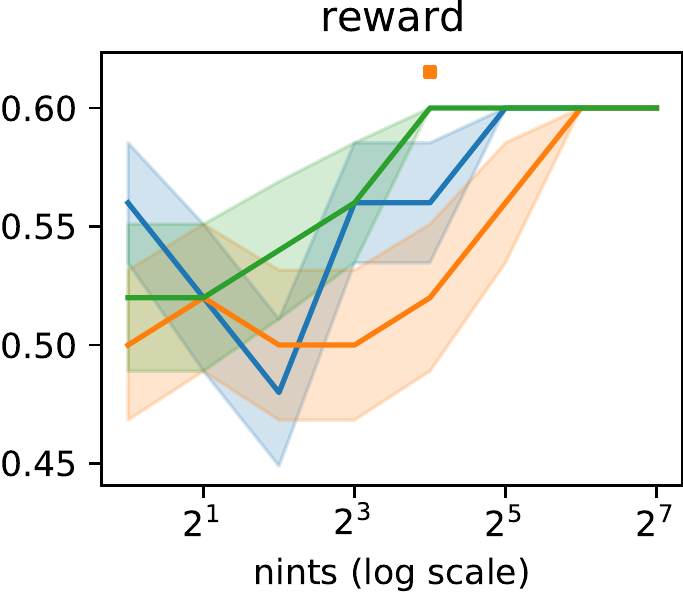}
    \hspace*{\fill}
    \includegraphics[height=2.8cm]{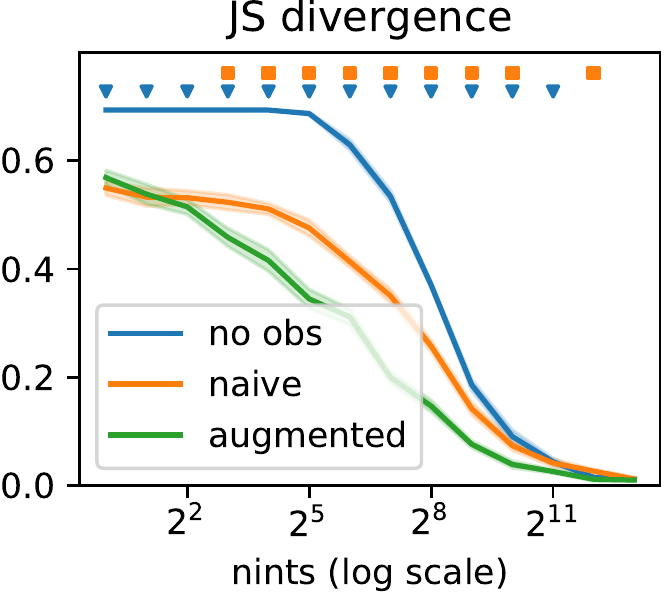}
    \hspace*{\fill}
    \includegraphics[height=2.8cm]{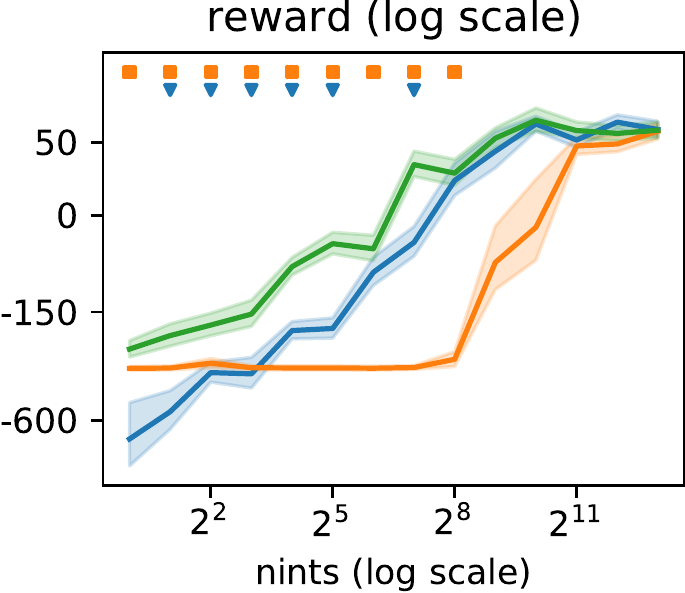}
    \caption{\textbf{Left}: our first toy problem \emph{door} (bandit). \textbf{Right}: our second toy problem \emph{tiger} (POMDP). We report both the mean and the standard error of the JS divergence (lower the better) and the expected reward (higher the better) over 10 seeds. In both cases the number of observational data is fixed to a large sample size (512), while the number of interventional data grows exponentially. Our augmented method shows the best sample-efficiency in both metrics. We report the significance of a Welch’s t-test ($\alpha<5\%$) versus the two other baselines \emph{no obs} and \emph{naive}, with triangle and square markers respectively.
    The high initial reward for \emph{no obs} in the bandit experiment can be attributed to the initial model parameters (prior) which have a high impact in the very low sample regime (1 or 2).
    }
    \label{fig:results}
\end{figure}

\section{Related work}
\label{sec:relatedworks}

A whole body of work exists around the question of merging interventional and observational data in RL. \citet{conf/nips/BareinboimFP15} study a sequential decision problem similar to ours, but assume that expert intentions are observed both in the interventional and the observational regimes, i.e., prior to doing interventions the learning agent can ask ``what would the expert do in my situation ?'' This introduces an intermediate, observed variable $\hat{a}_t = f(o_t)$ with the property that $p_\textit{prv}(a_t=\hat{a}_t|\hat{a}_t)=1$, which guarantees unconfoundedness in the observational regime ($A_t \indep S_t | H_t$), so that observational data can be considered interventional, and the standard PO-MDP transition model can be directly estimated via (\ref{eq:direct-learning}). \citet{conf/ijcai/ZhangB17,conf/aaai/ZhangB21} relax this assumption in the context of binary bandits, and later on in the more general context of dynamic treatment regimes \citep{conf/neurips/ZhangB19,conf/icml/ZhangB20}. They derive causal bounds similar to ours (Theorem~\ref{th:bounds}), and propose a two-step approach which first extracts causal bounds from observational data, and then uses these bounds in an online RL algorithm. While their method nicely tackles the question of leveraging observational data for online exploration, it does not account for uncertainty in the bounds estimated from the observational data. In comparison, our latent-based approach is more flexible, as it never computes explicit bounds, but rather lets the learning agent decide through (\ref{eq:latent-learning}) how data from both regimes influence the final transition model, depending of the number of samples available.
\citet{conf/nips/KallusPS18} also propose a two-step learning procedure to combine observational and interventional data in the context of binary contextual bandits. Their method however relies on a series of strong parametric assumptions (strong one-way overlap, linearity, non-singularity, finite fourth moment, strong overlapping).

A specific instantiation of our framework is off-policy evaluation, i.e., estimating the performance of a policy $\pi$ using observational data only. This corresponds to the specific setting $|\mathcal{D}_\textit{int}|=0$, where it can be shown that the causal transition model is in general not recoverable in the presence of confounding variables. Still, a growing body of literature studies the question under specific structural or parametric assumptions \citep{arxiv/LuSHLJ18,conf/aaai/TennenholtzSM20,conf/aistats/BennettKLM21}.

In the context of imitation learning, \citet{conf/nips/HaanJL19} attribute the issue of \emph{causal misidentification}, that is, ascribing the actions of an agent to the wrong explanatory variables, to confounding. We argue that this explanation is erroneous, since their imitated experts are trained in the standard POMDP setting (interventional regime). This reasoning supports \citet{arxiv/SpencerCVZB21}, who shows that \emph{causal misidentification} is simply a manifestation of \emph{covariate shift}.

\section{Discussions}
\label{sec:discussions}

In this paper we have presented a simple, generic method for combining interventional and observational (potentially confounded) data in model-based reinforcement learning for POMDPs. We have demonstrated that our method is correct and efficient in the asymptotic case (infinite observational data), and we have illustrated its effectiveness on two synthetic toy problems. One limitation of our method is that it adds an additional challenge on top of model-based RL, that of learning a latent-based transition model. Still, the recent success of discrete latent models for solving complex RL tasks \citep{conf/iclr/HafnerLNB21} or tasks in high-dimensional domains \citep{arxiv/RameshPGGVRCS21} lets us envision that this difficulty can be overcome in practice.
A first potential extension to our work could be to use offline data to guide online exploration, in a fashion similar to \citet{conf/ijcai/ZhangB17,conf/neurips/ZhangB19,conf/icml/ZhangB20,conf/aaai/ZhangB21}. We envision that our latent-based transition model could be updated each time a new interventional data comes in, and then existing exploration schemes for model-based RL could be used.
A second direct extension to our method is to consider that several agents are observed, each with its own privileged policy, leading to multiple observational regimes. This would lead, in the asymptotic case, to a stronger implicit regularizer for the causal transition model.
A third, obvious extension would be to develop a similar approach for model-free RL, maybe in the form of a value-function regularizer.
Finally, we hope that our work will help to bridge the gap between the RL and causality communities, and will convince the RL community that causality is an adequate tool to reason about observational data, which is plentiful in the world.

\section{Acknowledgements}

We thank David Berger for interesting discussions and for pointing us to relevant bodies of work.

This work was supported by the Canada First Research Excellence Fund (CFREF), Canada Excellence Research Chairs (CERC), Calcul Québec\footnote{https://www.calculquebec.ca}, Compute Canada\footnote{https://www.computecanada.ca}, and the DEpendable Explainable Learning (DEEL) french-canadian research project \footnote{https://www.deel.ai}.

\bibliography{arxiv_main}

\clearpage
\appendix
\section{Appendix}

\subsection{Recovering the standard POMDP transition model.}

Recovering $\hat{q}(o_{t+1}|h_t,a_t,i=1)$ can be done as follows:
\begin{align*}
    \hat{q}(o_{t+1}|h_t,a_t,i=1) &= \sum_{z_t}^{\mathcal{Z}} \hat{q}(z_t|h_t,i=1) \sum_{z_{t+1}}^{\mathcal{Z}} \hat{q}(z_{t+1}|z_t,a_t) \hat{q}(o_{t+1}|z_{t+1})
    \text{.}
\end{align*}
The second and third terms are readily available as components of the augmented POMDP model $\hat{q}$, while the first term can be recovered by unrolling a forward algorithm over the augmented DAG structure. First, we have
\begin{align*}
    \hat{q}(z_0,o_0|i=1) = \hat{q}(z_0) \hat{q}(o_0|z_0)
    \text{,} \\
    \hat{q}(z_0|h_0,i=1) = \frac{\hat{q}(z_0,o_0|i=1)}{\sum_{z_0}^\mathcal{Z} \hat{q}(z_0,o_0|i=1)}
    \text{.}
\end{align*}
Then, for every $t'$ from $0$ to $t-1$,
\begin{align*}
    \hat{q}(z_{t'+1},o_{t'+1}|h_{t'},a_{t'},i=1) &= \sum_{z_{t'}}^{\mathcal{Z}} \hat{q}(z_{t'}|h_{t'},i=1)  \hat{q}(z_{t'+1}|z_{t'},a_{t'})  \hat{q}(o_{t'+1}|z_{t'+1})
    \text{,} \\
    \hat{q}(z_{t'+1}|h_{t'+1},i=1) &= \frac{\hat{q}(z_{t'+1},o_{t'+1}|h_{t'},a_{t'},i=1)}{\sum_{z_{t'+1}}^{S} \hat{q}(z_{t'+1},o_{t'+1}|h_{t'},a_{t'},i=1)}
    \text{.}
\end{align*}

\subsection{Proof of Theorem~\ref{th:bounds}.}

\thbounds*

\begin{proof}[Proof of Theorem~\ref{th:bounds}.]
Consider $q(\tau,i)\in\mathcal{Q}$ any distribution that follows our augmented POMDP constraints. Then, for every $T \geq 1$ we have
\begin{align*}
    \prod_{t=0}^{T-1} q(a_{t}|h_{t},i) q(o_{t+1}|h_{t},a_{t},i) &= \frac{q(\tau|i)}{q(h_0|i)} \\
    &= \sum_{z_{0 \to T}}^{\mathcal{Z}^{T+1}} q(z_0|h_0,i) \prod_{t=0}^{T-1} q(a_{t},z_{t+1},o_{t+1}|z_{t},h_{t},i)
    \text{,}
\end{align*}
by using $A_t,Z_{t+1},O_{t+1} \indep Z_{0\to t-1} \mid Z_t, H_t,I$, which can be read via $d$-separation in the augmented POMDP DAG. Likewise, for every $t \geq 0$ we have
\begin{align*}
    q(o_{t+1}|h_{t},a_{t},i=1) =& \sum_{z_{t+1}}^{\mathcal{Z}} q(z_{t+1},o_{t+1}|h_{t},a_{t},i=1) \\
    =& \sum_{z_{t}}^{\mathcal{Z}} q(z_{t}|h_t,i=1) \sum_{z_{t+1}}^{\mathcal{Z}} q(z_{t+1},o_{t+1}|z_t,h_t,a_t,i=0)
    \text{,}
\end{align*}
by using $Z_t \indep A_t \mid H_t,I=1$ and $Z_{t+1},O_{t+1} \indep I \mid Z_{t},A_{t},H_{t}$. Then for every $t \geq 1$ we can further write
\begin{align*}
    q(o_{t+1}|h_{t},a_{t},i=1) =& \sum_{z_{t}}^{\mathcal{Z}} \frac{q(z_t,o_t|h_{t-1},a_{t-1},i=1)}{q(o_{t}|h_{t-1},a_{t-1},i=1)} \sum_{z_{t+1}}^{\mathcal{Z}} q(z_{t+1},o_{t+1}|z_t,h_t,a_t,i=0)
    \text{.}
\end{align*}
By recursively decomposing every $q(z_t,o_t|h_{t-1},a_{t-1},i=1)$ until $t=0$, and finally by using $Z_0 \indep I \mid H_0$, we obtain that for any $T \geq 1$
\begin{align*}
    \prod_{t=0}^{T-1} q(o_{t+1}|h_{t},a_{t},i=1) &= \sum_{z_{0 \to T}}^{\mathcal{Z}^{T+1}} q(z_0|h_0,i=0) \prod_{t=0}^{T-1} q(z_{t+1},o_{t+1}|z_t,a_t,h_t,i=0)
    \text{,}
\end{align*}
which can be re-expressed as
\begin{align*}
    \prod_{t=0}^{T-1} q(o_{t+1}|h_{t},a_{t},i=1) &= \sum_{a'_{0 \to T-1}}^{\mathcal{A}^{T}} \sum_{z_{0 \to T}}^{\mathcal{Z}^{T+1}} q(z_0|h_0,i=0) \prod_{t=0}^{T-1} q(a'_t|z_t,h_t,i=0) q(z_{t+1},o_{t+1}|z_t,h_t,a_t,i=0)
    \text{.}
\end{align*}
By considering the case $a'_{0\to T-1} = a_{0\to T-1}$ in isolation, and by assuming probabilities are positive, we readily obtain our first bound,
\begin{align*}
    \prod_{t=0}^{T-1} q(o_{t+1}|h_{t},a_{t},i=1)
    \geq \prod_{t=0}^{T-1} q(a_{t}|h_{t},i=0) q(o_{t+1}|h_{t},a_{t},i=0)
    \text{.}
\end{align*}
In order to obtain our second bound, we further isolate the cases $a'_0 \neq a_0$, then $a'_0 = a_0 \land a'_1 \neq a_1$, then $a'_{0} = a_{0} \land a'_{1} = a_{1} \land a'_2 \neq a_2$ and so on until $a'_{0\to{T-2}} = a_{0\to{T-2}} \land a'_{T-1} \neq a_{T-1}$, which yields
\begin{align*}
    &\prod_{t=0}^{T-1} q(o_{t+1}|h_{t},a_{t},i=1)
    = \prod_{t=0}^{T-1} q(a_{t}|h_{t},i=0) q(o_{t+1}|h_{t},a_{t},i=0) \\
    +& \sum_{z_{0 \to T}}^{\mathcal{Z}^{T+1}} q(z_0|h_0,i=0) \left(1 - q(a_0|z_0,h_0,i=0)\right) \prod_{t=0}^{T-1} q(z_{t+1},o_{t+1}|z_t,h_t,a_t,i=0) \\
    +& \sum_{K=0}^{T-2} \sum_{z_{0 \to T}}^{\mathcal{Z}^{T+1}} q(z_0|h_0,i=0) \prod_{t=0}^{K} q(a_t,z_{t+1},o_{t+1}|z_t,h_t,i=0) \left(1 - q(a_K|z_K,h_K,i=0)\right) \\ & \prod_{t=K+1}^{T-1} q(z_{t+1},o_{t+1}|z_t,h_t,a_t,i=0)
    \text{.}
\end{align*}
Then by assuming probabilities are upper bounded by 1, we obtain
\begin{align*}
    \prod_{t=0}^{T-1} q(o_{t+1}|h_{t},a_{t},i=1)
    &\leq \prod_{t=0}^{T-1} q(a_{t}|h_{t},i=0) q(o_{t+1}|h_{t},a_{t},i=0) + 1 - q(a_0|h_0,i=0) \\
    +& \sum_{K=0}^{T-2} \prod_{t=0}^{K} q(o_{t+1}|h_t,a_t,i=0) \left(\prod_{t=0}^{K-1} q(a_t|h_t,i=0) - \prod_{t=0}^{K} q(a_t|h_t,i=0)\right) \\
    &\leq \prod_{t=0}^{T-1} q(a_{t}|h_{t},i=0) q(o_{t+1}|h_{t},a_{t},i=0) + 1 - \prod_{t=0}^{T-1} q(a_t|h_t,i=0)
    \text{.}
\end{align*}
Finally, with $\hat{q}$ solution of (\ref{eq:latent-learning}) and $|\mathcal{D}_{obs}| \to \infty$ we have that $\displaystyle D_{\mathrm{KL}} ( p(\tau|i=0) \Vert \hat{q}(\tau|i=0) ) = 0$, and thus $\hat{q}(a_{t}|h_{t},i=0) = p(a_{t}|h_{t},i=0)$ and $\hat{q}(o_{t+1}|h_{t},a_{t},i=0) = p(o_{t+1}|h_{t},a_{t},i=0)$, which shows the desired result.
\end{proof}

\clearpage
\subsection{Experimental details}
\label{sec:expe_details}

\paragraph{Training.} In all our experiments we use a latent space $\mathcal{Z}$ size of $|\mathcal{Z}|=32$ and train our tabular model $\hat{q}$ by minimizing the negative log likelihood via gradient descent. We use the Adam optimizer \citep{conf/iclr/KingmaB14} with a learning rate of $10^{-2}$, and train for $500$ epochs consisting of $50$ gradient descent steps with minibatches of size $32$. We divide the learning rate by 10 after 10 epochs without loss improvement (reduce on plateau), and we stop training after 20 epochs without improvement (early stopping). In the first experiment we derive the optimal policy $\hat{\pi}^\star$ exactly, and in the second experiment we train a ``dreamer'' RL agent on imaginary samples $\tau \sim \hat{q}(\tau|i=1)$ obtained from the model, using the belief states $\hat{q}(s_t|h_t)$ as features. We use a simple Actor-Critic algorithm for training, and our agents consist of a simple MLP with two hidden layers for both the critic and the policy parts. Agents are trained until convergence or with a maximum number of 1000 epochs, with a learning rate of $10^{-2}$, a discount factor $\gamma=0.9$ and a batch size of $8$.

\paragraph{JS divergence.} To evaluate the general quality of the recovered transition models, we compute the expected Jensen-Shannon divergence between the learned $\hat{q}(o_{t+1}|h_t,i=1)$ and the true $p(o_{t+1}|h_t,i=1)$, over transitions generated using a uniformly random policy $\pi_{\textit{rand}}$,
\begin{align*}
    &\frac{1}{2} \mathbb{E}_{\tau \sim p_\textit{init}, p_\textit{trans}, p_\textit{obs},\pi_\textit{rand}}\left[\log \frac{p(o_{0})}{m(o_{0})} + \sum_{t=1}^{|\tau|} \log \frac{p(o_{t+1}|h_t,i=1)}{m(o_{t+1}|h_t,i=1)}\right]\\
    + &\frac{1}{2}\mathbb{E}_{\tau \sim \hat{q}_\textit{init}, \hat{q}_\textit{trans}, \hat{q}_\textit{obs},\pi_\textit{rand}}\left[\log \frac{\hat{q}(o_{0})}{m(o_{0})} + \sum_{t=1}^{|\tau|} \log \frac{\hat{q}(o_{t+1}|h_t,i=1)}{m(o_{t+1}|h_t,i=1)}\right]
    \text{,}
\end{align*}
where $m(.) = \frac{1}{2}\left(\hat{q}(.) + p(.)\right)$. In the first experiment we compute the JS exactly, while in the second experiment we use a stochastic approximation over 100 trajectories $\tau$ to estimate each of the expectation terms in the JS empirically.

\paragraph{Reward.} To evaluate quality of the recovered transition models for solving the original RL task, that is, maximizing the expected long-term reward, we evaluate the policy $\hat{\pi}^\star$, obtained after planning with the recovered model $\hat{q}$, on the true environment $p$,
\begin{equation*}
    \mathbb{E}_{\tau \sim p_\textit{init}, p_\textit{trans}, p_\textit{obs},\hat{\pi}^\star}\left[\sum_{t=0}^{|\tau|} R(o_t)\right]
    \text{.}
\end{equation*}
In the first experiment we compute this expectation exactly, while in the second experiment we use a stochastic approximation using 100 trajectories $\tau$.

\clearpage
\subsection{Complete empirical results}
\label{sec:additional_results}

\subsubsection{Door experiment}
\label{sec:door_expe}

The \emph{door} experiment corresponds to a simple binary bandit setting, that is, a specific POMDP with horizon $H=1$ and no observation. The bandit dynamics are described in Table~\ref{tab:bandit_tables}.

\begin{table}[ht]
\begin{minipage}[t][3.5cm][t]{0.45\textwidth} 
    \centering
    \begin{tabular}{|c|c|}
    \hline
    \multicolumn{2}{|c|}{\textit{light}}     \\
    \hline
    red & green \\
    \hline
    0.6 & 0.4 \\
    \hline
    \multicolumn{2}{c}{} \\
    \multicolumn{2}{c}{$p(\textit{light})$}
    \end{tabular}
\end{minipage}
\begin{minipage}[t][3.5cm][t]{0.45\textwidth} 
    \centering
    \begin{tabular}{|c|c|c|c|}
    \cline{3-4}
    % \hline
    \multicolumn{2}{c}{} & \multicolumn{2}{|c|}{\textit{door}} \\
    \hline
    \textit{light} & \textit{button} & closed  & open \\
    \hline
    \multirow{2}{*}{red} & A & 0.0 & 1.0 \\
    \cline{2-4}
    & B & 1.0 & 0.0 \\
    \hline
    \multirow{2}{*}{green} & A & 1.0 & 0.0 \\
    \cline{2-4}
    & B & 0.0 & 1.0 \\
    \hline
    \multicolumn{4}{c}{} \\
    \multicolumn{4}{c}{$p(\textit{door}|\textit{light}, \textit{button})$}
    \end{tabular}
\end{minipage}
\caption{Probability tables for our \emph{door} bandit problem.}
\label{tab:bandit_tables}
\end{table}

We repeat the \emph{door} experiment in six different scenarios, corresponding to different privileged policies $\pi_\textit{prv}$ ranging from a totally random agent to a perfectly good and a perfectly bad agent. Each time, we evaluate the performance of the \emph{no obs}, \emph{naive} and \emph{augmented} approaches under different data regimes, by varying the sample size for both the observational data $\mathcal{D}_\textit{obs}$ and the interventional data $\mathcal{D}_\textit{int}$ in the range $(1, 2, 4, 8, 16, 32, 64, 128, 256, 512)$.

In each scenario, we report both the expected reward and the JS as heatmaps with $|\mathcal{D}_\textit{int}|$ and $|\mathcal{D}_\textit{obs}|$ in the $x$-axis and $y$-axis respectively, to highlight the combined effect of the sample sizes on each approach. We also provide as a heatmap the difference between our approach, \emph{augmented}, and the two other approaches \emph{no obs} and \emph{naive}. We always plot the expected reward in the first row, and JS in the second row. As a remark, shades of green show gains in reward (the higher the better), while shades of purple show gains in JS (the lower the better).

Finally, we also present two plots which provide a focus on the data regime that corresponds to the largest number of observational data ($|\mathcal{D}_\textit{obs}|=512$), as in the main paper.

The scenario reported in the main paper is the first one, that is, \emph{noisy good agent}.

% Similarly to the scenario considered throughout the article in the \textit{guiding example}, one will find a description of the agent's policy - what button to press given the observed light, alongside the Jensen-Shannon (JS) divergence between the recovered $\hat{q}(o_{t+1}|h_t,a_t,i=1)$ and the true transition model $p(o_{t+1}|h_t,do(a_t))$, and the expected reward obtained when the model is used for planning. Those two figures are obtained using a fixed and large amount of observational data (512), while the number of interventional data grows exponentially.

% In addition, we provide a grid of ten heatmaps. On the three first columns of the first row (\textit{respectively} second row), we plot the expected reward obtained when the model is used for planning (\textit{respectively} the JS divergence between the recovered transition model and the true transition model) on a average ten seeds, by varying both the number of interventional \textit{and} observational data in a exponential manner. Eventually, the last two columns are obtained by subtracting results from the \textit{augmented} model minus the \textit{no obs} model on the one hand, and  the \textit{augmented} model minus the \textit{naive obs+int} model on the other hand. They aimed at showing the different regime of combinations of number of observational and interventional data where our method brings gains, be it in JS or in reward. As a remark, shades of green (\textit{respectively} purple) shows gains in reward (\textit{respectively} in JS).

\clearpage
\textbf{Noisy Good Agent}

In the noisy good agent setting, the agent plays halfway between a perfect and a random policy. The diversity of its action leads to a good start for the \textit{naive} model but the bias it contains is hard to overcome. In contrast, our method makes good use of the observational data from the start and is also able to correct the bias as interventional data come in, eventually converging towards the true transition model. One can clearly see two regimes where our approach takes the best of \textit{naive} at first, and \textit{no obs} then - see Figure~\ref{fig:toy1_noisy_good}.

\begin{figure}[h!]
    \centering
    
    \raisebox{22mm}{
    \begin{minipage}{0.3\textwidth} 
        \centering
        % \captionof{table}{Agent's policy}
        \begin{tabular}{|c|c|c|}
        \multicolumn{3}{c}{$\pi_\textit{prv}(\textit{button}|\textit{light})$}\vspace{0.05 in}\\
        \cline{2-3}
        \multicolumn{1}{c}{} & \multicolumn{2}{|c|}{\textit{button}}     \\
        \hline
        \textit{light}  & A  & B   \\
        \hline
        red           & 0.9       & 0.1    \\
        green           & 0.4        & 0.6    \\
        \hline
        \end{tabular}
    \end{minipage}}
    \hspace*{\fill}
    \includegraphics[width=0.3\textwidth]{toy1/noisy_good/js_nobs_512.pdf}
   \hspace*{\fill}
    \includegraphics[width=0.3\textwidth]{toy1/noisy_good/reward_nobs_512.pdf} \\
    \hspace*{\fill} \\
    \includegraphics[width=\textwidth]{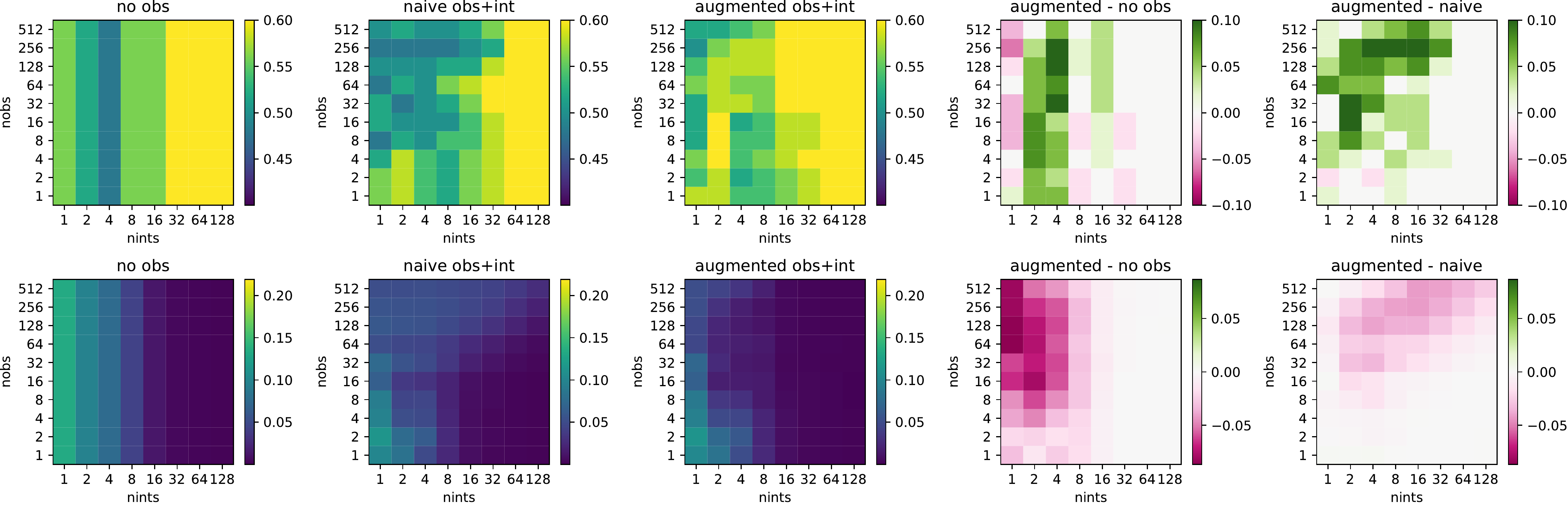}
    \caption{Noisy good agent setting. Heatmaps correspond respectively to the expected reward (top row, higher is better) and the JS divergence (bottom row, lower is better).}
    \label{fig:toy1_noisy_good}
\end{figure}

\clearpage
\textbf{Random Agent}

A random policy yields by essence unconfounded observational data, as it does not exploit the privileged information. Hence, the \textit{naive} approach is unbiased in this case, and makes effective use of the observational data. Our approach, \emph{augmented}, exhibits an overall comparable performance, only slightly worse at times. We believe this can be explained by the additional complexity of our method, which goes through a deconfounding step, and is not best suited to random data - see Figure~\ref{fig:toy1_random}.

\begin{figure}[h]
    \centering
    
    \raisebox{22mm}{
    \begin{minipage}{0.3\textwidth} 
        \centering
        % \captionof{table}{Agent's policy}
        \begin{tabular}{|c|c|c|}
        \multicolumn{3}{c}{$\pi_\textit{prv}(\textit{button}|\textit{light})$}\vspace{0.05 in}\\
        \cline{2-3}
        \multicolumn{1}{c}{} & \multicolumn{2}{|c|}{\textit{button}}     \\
        \hline
        \textit{light}  & A  & B   \\
        \hline
        red           & 0.5        & 0.5    \\
        green           & 0.5        & 0.5    \\
        \hline
        \end{tabular}
    \end{minipage}}
    \hspace*{\fill}
    \includegraphics[width=0.3\textwidth]{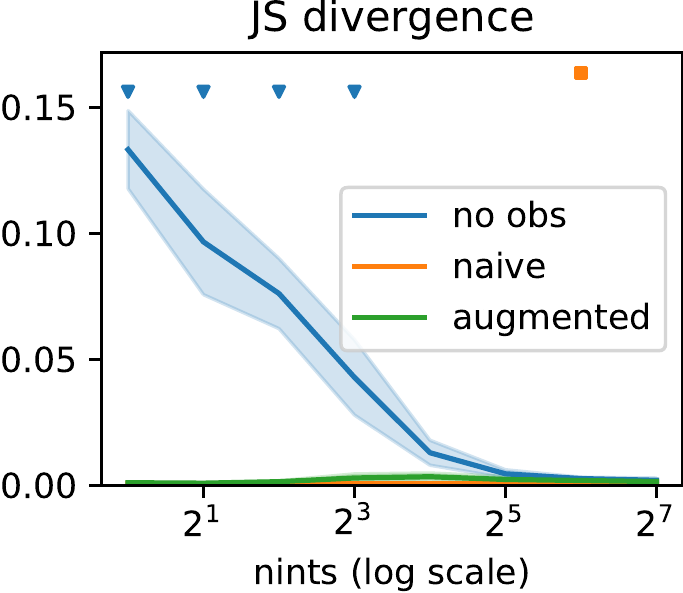}
    \hspace*{\fill}
    \includegraphics[width=0.3\textwidth]{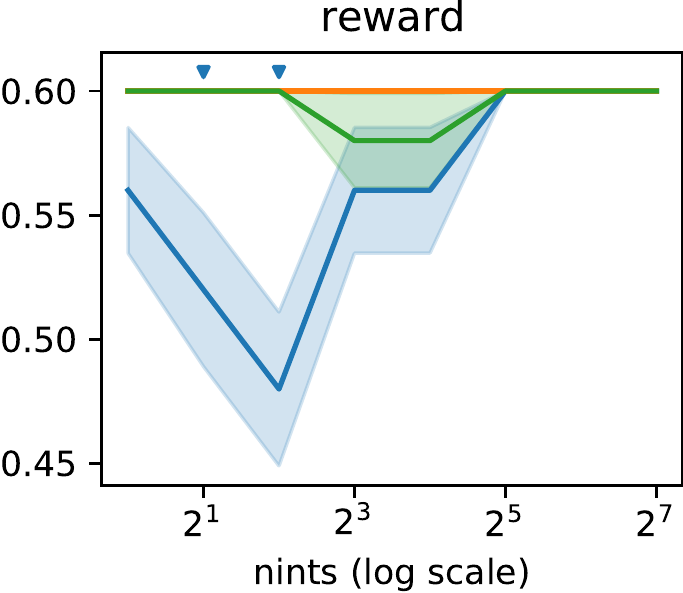}
    \hspace*{\fill} \\
    \includegraphics[width=\textwidth]{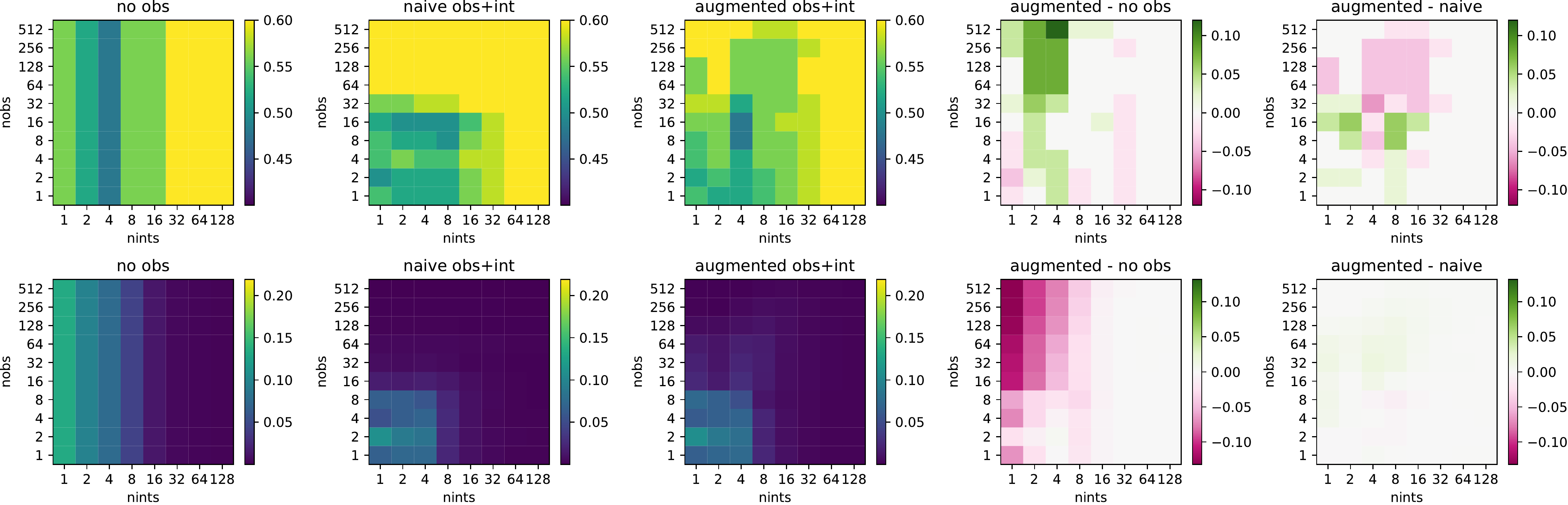}
    \caption{Random agent setting. Heatmaps correspond respectively to the expected reward (top row, higher is better) and the JS divergence (bottom row, lower is better).}
    \label{fig:toy1_random}
\end{figure}

\clearpage
\textbf{Perfectly Good Agent}

Observing a perfectly good agent playing induces a strong positive bias, because every observed action always collects a positive reward. As such, the \textit{naive} approach struggles to learn a good transition model. The bias however is quickly corrected by our \emph{augmented} approach, which eventually converges to the true transition model faster than the \emph{no obs} approach - see Figure ~\ref{fig:toy1_perfect_good}.

\begin{figure}[ht]
    \centering
    
    \raisebox{22mm}{
    \begin{minipage}{0.3\textwidth} 
        \centering
        % \captionof{table}{Agent's policy}
        \begin{tabular}{|c|c|c|}
        \multicolumn{3}{c}{$\pi_\textit{prv}(\textit{button}|\textit{light})$}\vspace{0.05 in}\\
        \cline{2-3}
        \multicolumn{1}{c}{} & \multicolumn{2}{|c|}{\textit{button}}     \\
        \hline
        \textit{light}  & A  & B   \\
        \hline
        red           & 1.0       & 0.0    \\
        green           & 0.0        & 1.0    \\
        \hline
        \end{tabular}
    \end{minipage}}
    \hspace*{\fill}
    \includegraphics[width=4cm]{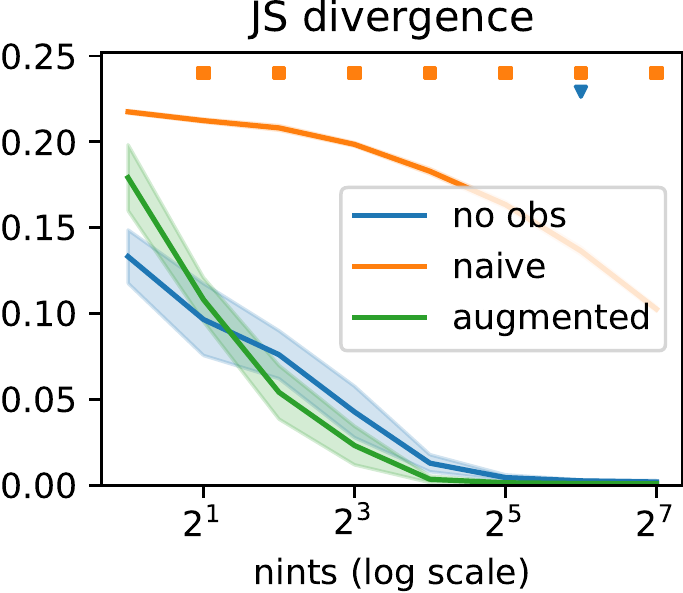}
    \hspace*{\fill}
    \includegraphics[width=4cm]{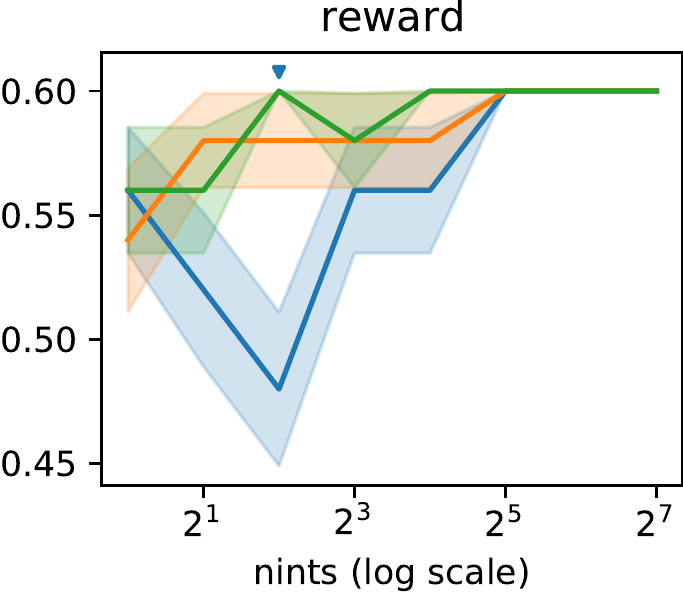}
    \hspace*{\fill} \\
    \includegraphics[width=\textwidth]{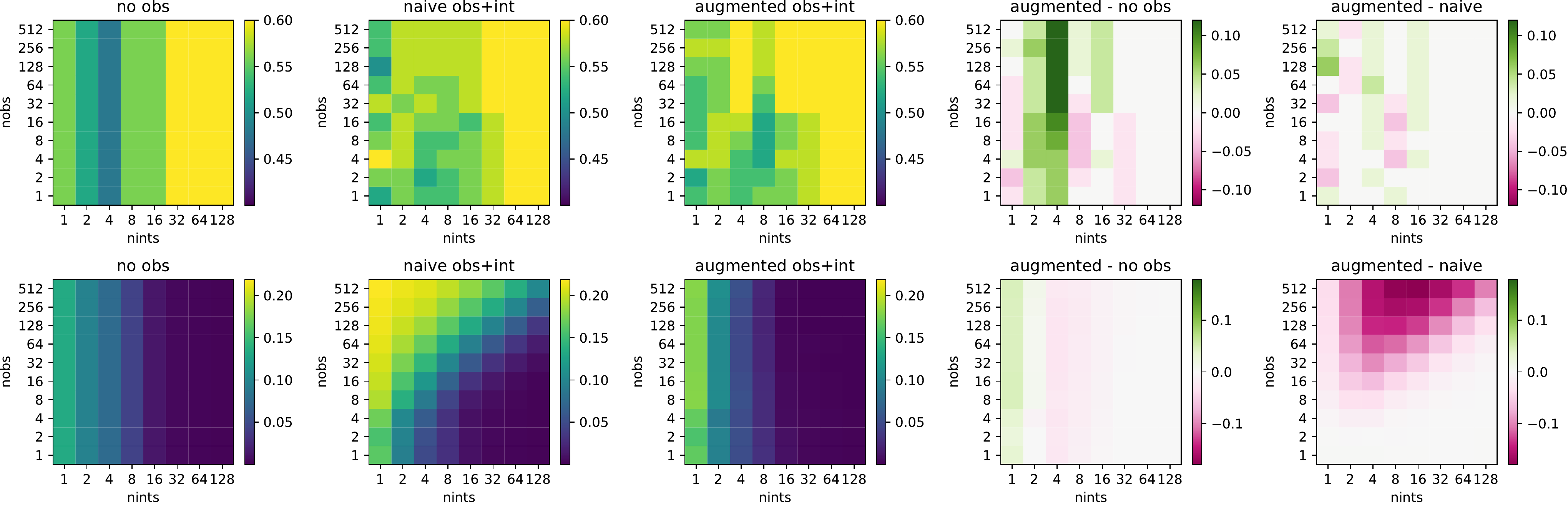}
    \caption{Perfectly good agent setting. Heatmaps correspond respectively to the expected reward (top row, higher is better) and the JS divergence (bottom row, lower is better).}
    \label{fig:toy1_perfect_good}
\end{figure}

\clearpage
\textbf{Perfectly Bad Agent}

Similarly to the previous setting, observing an agent that always chooses a bad action leads to a strong negative bias, as every action is associated to a low reward. The behaviour in terms of JS and reward is similar as well, however our approach is not clearly distinguishable from \emph{no obs} in this setting - see Figure~\ref{fig:toy1_perfect_bad}.

\begin{figure}[ht]
    \centering
    
    \raisebox{22mm}{
    \begin{minipage}{0.3\textwidth} 
        \centering
        % \captionof{table}{Agent's policy}
        \begin{tabular}{|c|c|c|}
        \multicolumn{3}{c}{$\pi_\textit{prv}(\textit{button}|\textit{light})$}\vspace{0.05 in}\\
        \cline{2-3}
        \multicolumn{1}{c}{} & \multicolumn{2}{|c|}{\textit{button}}     \\
        \hline
        \textit{light}  & A  & B   \\
        \hline
        red           & 0.0       & 1.0    \\
        green           & 1.0        & 0.0    \\
        \hline
        \end{tabular}
    \end{minipage}}
    \hspace*{\fill}
    \includegraphics[width=4cm]{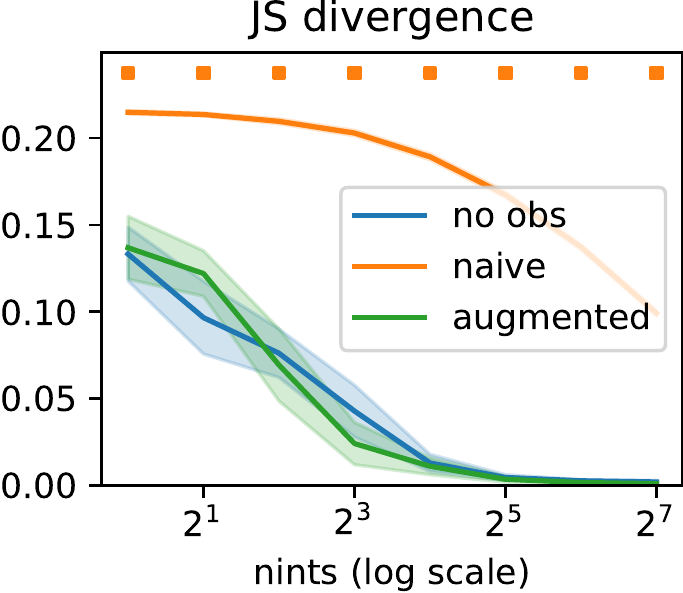}
    \hspace*{\fill}
    \includegraphics[width=4cm]{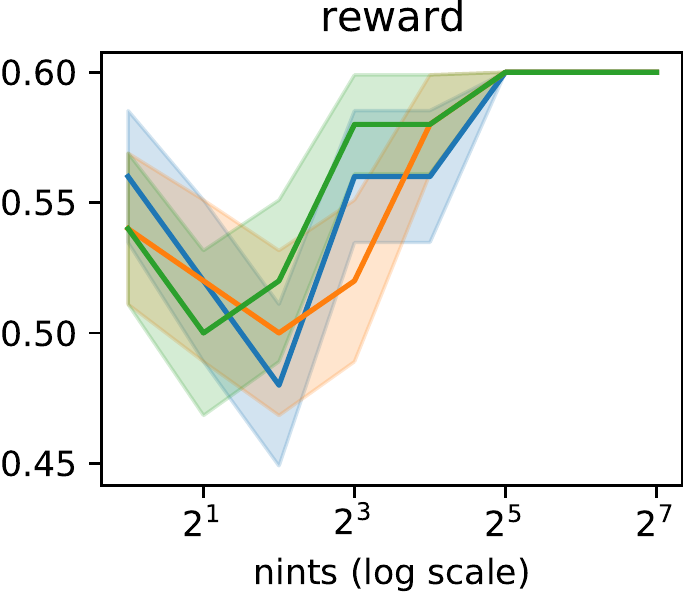}
    \hspace*{\fill} \\
    \includegraphics[width=\textwidth]{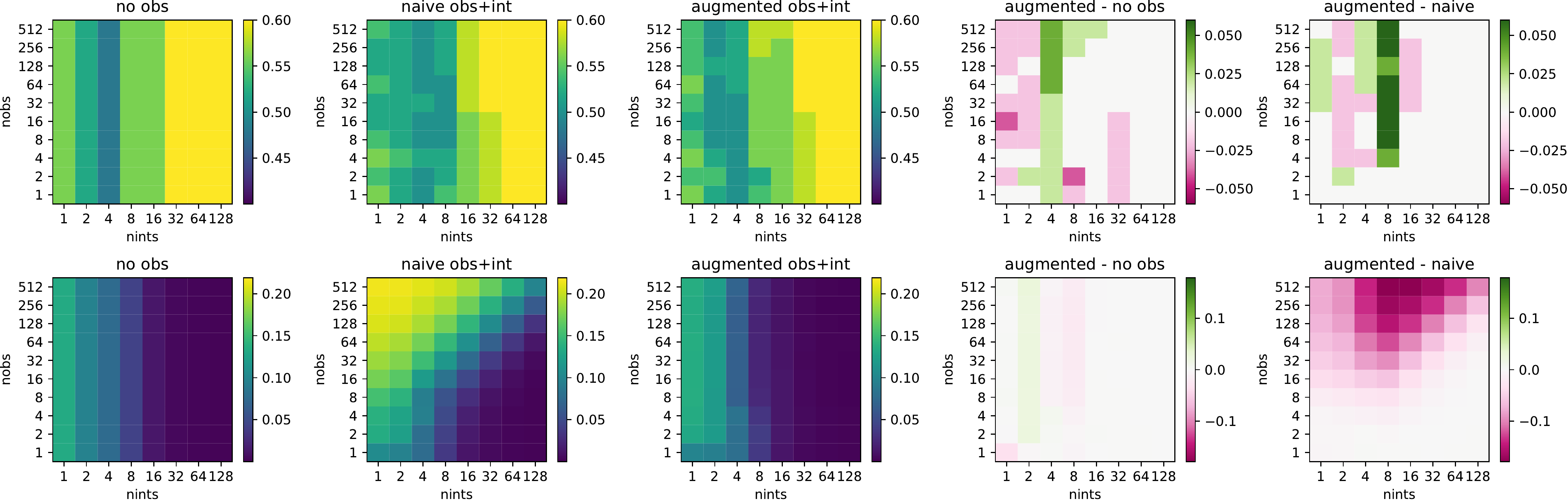}
    \caption{Perfectly bad agent setting. Heatmaps correspond respectively to the expected reward (top row, higher is better) and the JS divergence (bottom row, lower is better).}
    \label{fig:toy1_perfect_bad}
\end{figure}

\clearpage
\textbf{Positively Biased Agent}

Here the agent's policy is considered as \textit{positively biased} in the sense that the agent will only obtain a positive reward when playing button A (with 55\% chances) and never by playing button B (0\% chances). Because playing button A is actually the optimal policy, this strong bias has a positive effect on the reward. Hence the \textit{naive} approach, although worse in terms of JS that our approach, will always lead to a very good policy in terms of reward. This can easily be witnessed in the \textit{augmented - naive} reward plot of Figure~\ref{fig:toy1_strong_good_bias}, where purple reflects a deterioration.

\begin{figure}[ht]
    \centering
    
    \raisebox{22mm}{
    \begin{minipage}{0.3\textwidth} 
        \centering
        % \captionof{table}{Agent's policy}
        \begin{tabular}{|c|c|c|}
        \multicolumn{3}{c}{$\pi_\textit{prv}(\textit{button}|\textit{light})$}\vspace{0.05 in}\\
        \cline{2-3}
        \multicolumn{1}{c}{} & \multicolumn{2}{|c|}{\textit{button}}     \\
        \hline
        \textit{light}  & A  & B   \\
        \hline
        red           & 0.8       & 0.2    \\
        green           & 1.0       & 0.0   \\
        \hline
        \end{tabular}
    \end{minipage}}
    \hspace*{\fill}
    \includegraphics[width=4cm]{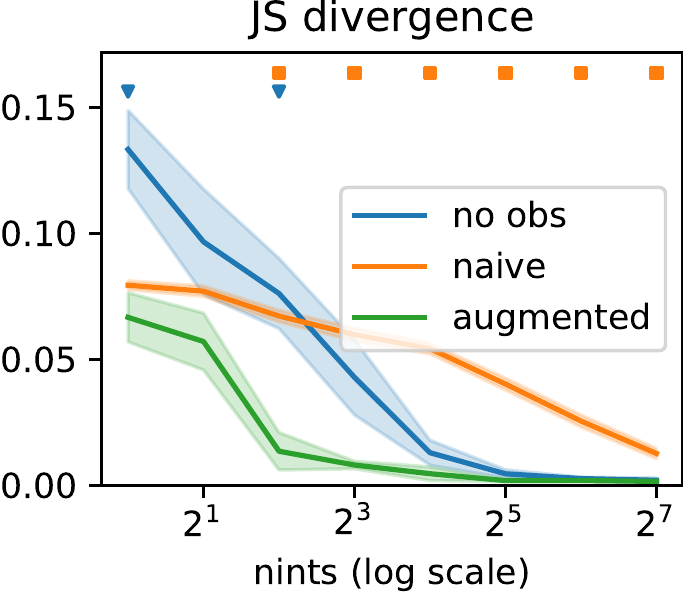}
    \hspace*{\fill}
    \includegraphics[width=4cm]{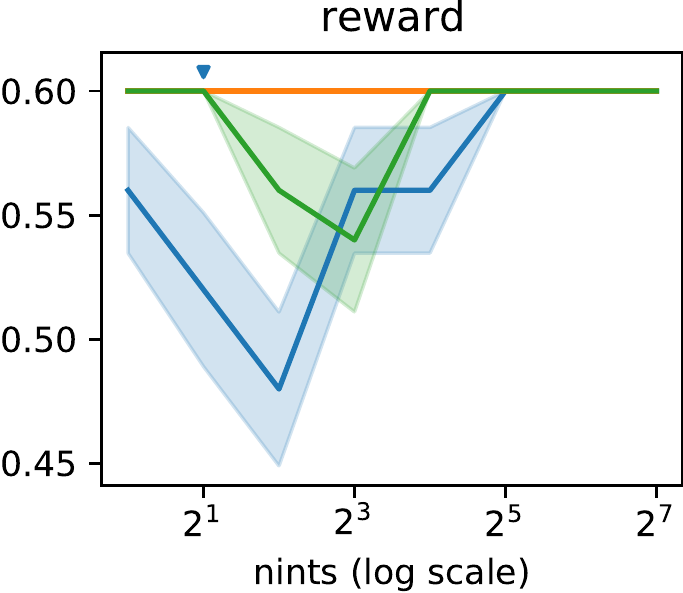}
    \hspace*{\fill} \\
    \includegraphics[width=\textwidth]{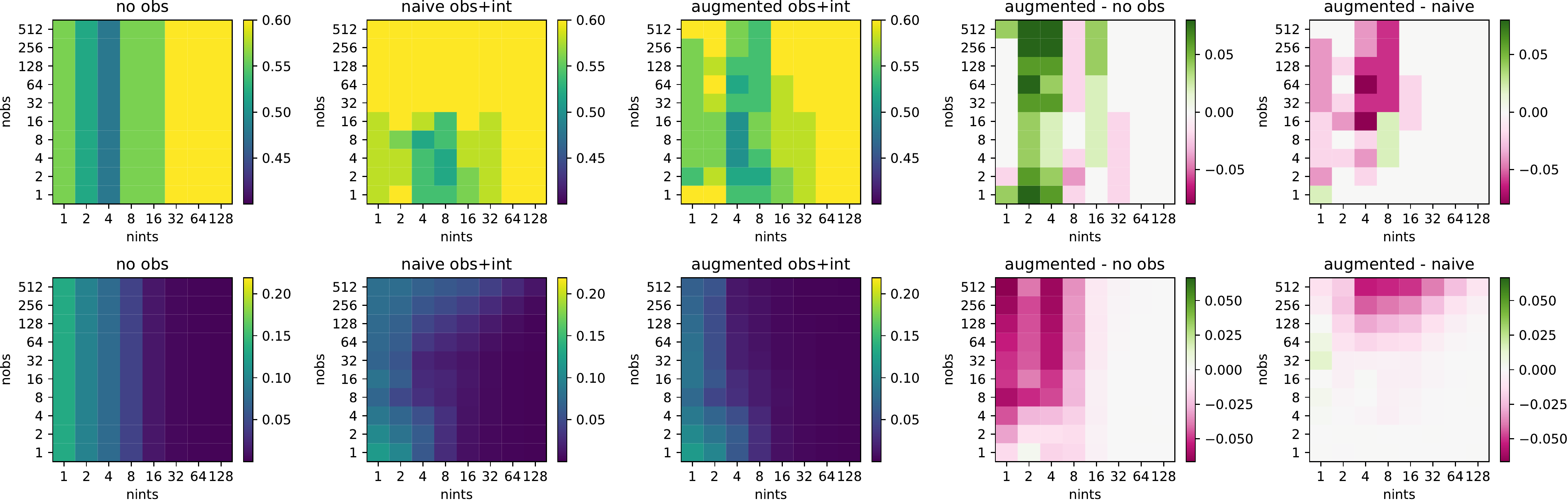}
    \caption{Positively biased agent setting. Heatmaps correspond respectively to the expected reward (top row, higher is better) and the JS divergence (bottom row, lower is better).}
    \label{fig:toy1_strong_good_bias}
\end{figure}

\clearpage
\textbf{Negatively Biased Agent}\\

In an analogous way, a negatively biased agent will overuse button A, leading to mixed feelings regarding this button, whereas it will always get a positive reward each time it uses button B. This leads to the opposite behavior as we had in the previous setting, with the \textit{naive} approach always favoring the use of button B, and obtaining a bad performance in terms of reward. The \emph{naive} approach only gets better when a lot of interventional data is combined with the biased observational data, while our \emph{augmented} approach is able to overcome this pessimistic bias very early on, and converges faster than both \emph{no obs} and \emph{naive} - see Figure ~\ref{fig:toy1_strong_bad_bias}.

\begin{figure}[ht]
    \centering
    
    \raisebox{22mm}{
    \begin{minipage}{0.3\textwidth} 
        \centering
        % \captionof{table}{Agent's policy}
        \begin{tabular}{|c|c|c|}
        \multicolumn{3}{c}{$\pi_\textit{prv}(\textit{button}|\textit{light})$}\vspace{0.05 in}\\
        \cline{2-3}
        \multicolumn{1}{c}{} & \multicolumn{2}{|c|}{\textit{button}}     \\
        \hline
        \textit{light}  & A  & B   \\
        \hline
        red           & 1.0       & 0.0    \\
        green           & 0.8        & 0.2   \\
        \hline
        \end{tabular}
    \end{minipage}}
    \hspace*{\fill}
    \includegraphics[width=4cm]{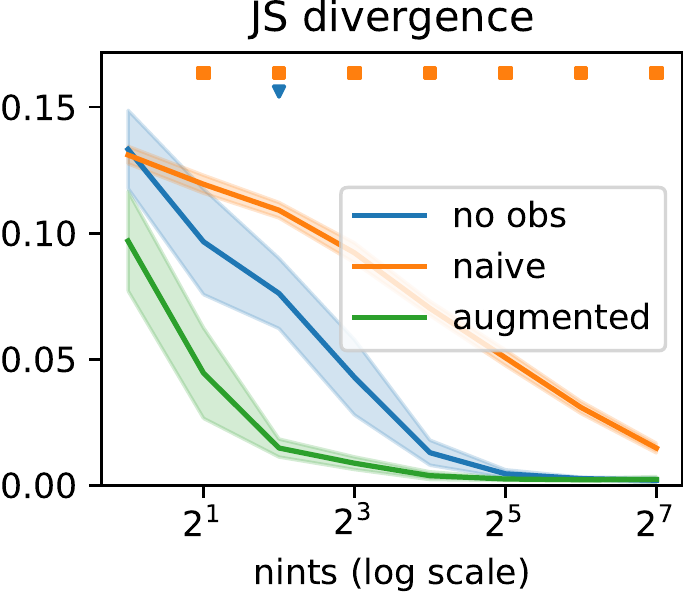}
    \hspace*{\fill}
    \includegraphics[width=4cm]{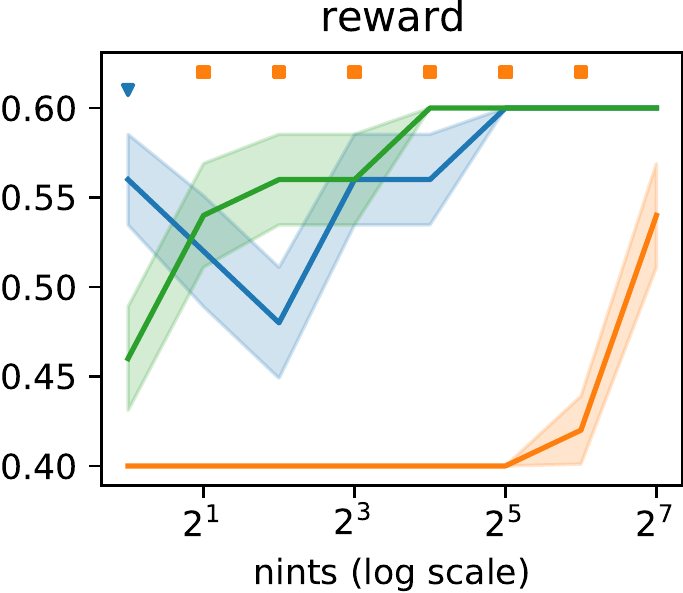}
    \hspace*{\fill} \\
    \includegraphics[width=\textwidth]{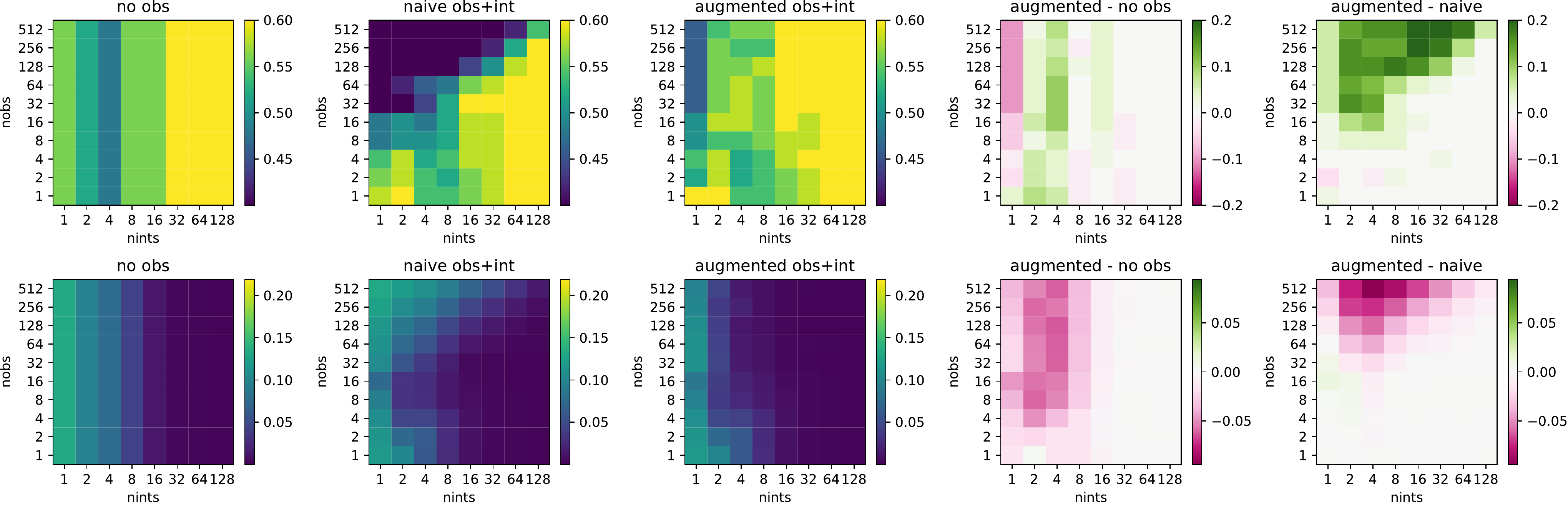}
    \caption{Pessimistic bias agent setting. Heatmaps correspond respectively to the expected reward (top row, higher is better) and the JS divergence (bottom row, lower is better).}
    \label{fig:toy1_strong_bad_bias}
\end{figure}

\clearpage
\subsubsection{Tiger experiment}

The \emph{tiger} experiment corresponds a synthetic POMDP toy problem proposed by \citet{conf/aaai/CassandraKL94}. In short, in this problem the agent stands in front of two doors to open, one of them having a tiger behind it (-100 reward), and the other one a treasure (+10 reward). The agent also gets a noisy observation of the system in the form of the roar from the tiger, which seems to originate from the correct door most of the time (85\% chances) and the wrong door sometimes (15\% chances). In order to reduce uncertainty the agent can listen to the tiger's roar again, at the cost of a small penalty (-1). We present the simplified POMDP dynamics in Table~\ref{tab:tiger_tables}, and in our experiments we impose a fixed horizon of size $H=50$.

\begin{table}[ht]
\centering
\begin{minipage}[t][7cm][t]{0.45\textwidth} 
    \centering
    \begin{tabular}{|c|c|}
    \hline
    \multicolumn{2}{|c|}{$\textit{tiger}$}     \\
    \hline
    left & right \\
    \hline
    0.5 & 0.5 \\
    \hline
    \multicolumn{2}{c}{} \\
    \multicolumn{2}{c}{$p(\textit{tiger})$}
    \end{tabular}
    
    \vspace{0.5cm}
    \begin{tabular}{|c|c|c|c|}
    \cline{3-4}
    % \hline
    \multicolumn{2}{c}{} & \multicolumn{2}{|c|}{$\textit{tiger}_{t+1}$} \\
    \hline
    $\textit{tiger}_t$ & $\textit{action}_t$ & left & right \\
    \hline
    \multirow{3}{*}{left} & listen & 1.0 & 0.0 \\
    \cline{2-4}
    & open left & 0.5 & 0.5 \\
    \cline{2-4}
    & open right & 0.5 & 0.5 \\
    \hline
    \multirow{3}{*}{right} & listen & 0.0 & 1.0 \\
    \cline{2-4}
    & open left & 0.5 & 0.5 \\
    \cline{2-4}
    & open right & 0.5 & 0.5 \\
    \hline
    \multicolumn{4}{c}{} \\
    \multicolumn{4}{c}{$p(\textit{tiger}_{t+1}|\textit{tiger}_t, \textit{action}_t)$}
    \end{tabular}
\end{minipage}
\begin{minipage}[t][7cm][t]{0.45\textwidth} 
    \centering
    \begin{tabular}{|c|c|c|}
    \cline{2-3}
    \multicolumn{1}{c}{} & \multicolumn{2}{|c|}{$\textit{roar}$}     \\
    \hline
    $\textit{tiger}$ & left & right \\
    \hline
    left & 0.85 & 0.15 \\
    \hline
    right & 0.15 & 0.85 \\
    \hline
    \multicolumn{3}{c}{} \\
    \multicolumn{3}{c}{$p(\textit{roar} | \textit{tiger})$}
    \end{tabular}
    
    \vspace{0.5cm}
    \begin{tabular}{|c|c|c|c|c|}
    \cline{3-5}
    % \hline
    \multicolumn{2}{c}{} & \multicolumn{3}{|c|}{$\textit{reward}$} \\
    \hline
    $\textit{tiger}$ & $\textit{action}$ & -1 & -100 & +10 \\
    \hline
    \multirow{3}{*}{left} & listen & 1.0 & 0.0 & 0.0 \\
    \cline{2-5}
    & open left & 0.0 & 1.0 & 0.0 \\
    \cline{2-5}
    & open right & 0.0 & 0.0 & 1.0 \\
    \hline
    \multirow{3}{*}{right} & listen & 1.0 & 0.0 & 0.0 \\
    \cline{2-5}
    & open left & 0.0 & 0.0 & 1.0 \\
    \cline{2-5}
    & open right & 0.0 & 1.0 & 0.0 \\
    \hline
    \multicolumn{5}{c}{} \\
    \multicolumn{5}{c}{$p(\textit{reward}|\textit{tiger}, \textit{action})$}
    \end{tabular}
\end{minipage}
\caption{Probability tables for the \emph{tiger} problem.}
\label{tab:tiger_tables}
\end{table}

For the tiger experiment we again consider again six different privileged policies $\pi_\textit{prv}$ for the observed agent. We then evaluate the performance of the \emph{no obs}, \emph{naive} and \emph{augmented} approaches under different data regimes, by keeping the observational data fixed to $|\mathcal{D}_\textit{obs}|=512$ while varying the varying the number of interventional data for $\mathcal{D}_\textit{int}$ in the range $(1, 2, 4, 8, 16, 32, 64, 128, 256, 512, 1024, 2048, 4096, 8192)$.

The scenario reported in the main paper is the first one, that is, \emph{noisy good agent}.

\clearpage
\textbf{Noisy Good Agent}\\

In this scenario the privileged agent adopts a policy which plays the optimal action most of the time, but also sometimes decides to listen or to open the wrong door. As can be seen, in this scenario our \emph{augmented} method makes the best use of the observational data, and is significantly better than both the \emph{no obs} and \emph{naive} approaches in the low-sample regime.

\begin{figure}[ht]
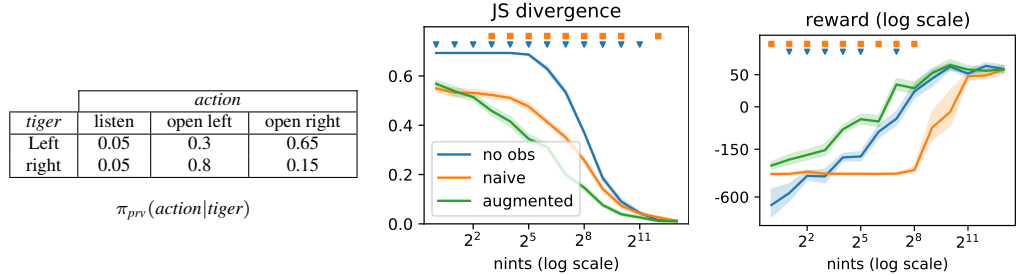

    \hspace*{\fill}
    \raisebox{17mm}{
    % \footnotesize
    \scriptsize
    \begin{tabular}{|c|c|c|c|}
        \multicolumn{3}{c}{}\vspace{0.05 in}\\
        \cline{2-4}
        \multicolumn{1}{c}{} & \multicolumn{3}{|c|}{\textit{action}}     \\
        \hline
        \textit{tiger}  & listen & open left & open right \\
        \hline
        Left       & 0.05 & 0.3 & 0.65 \\
        right      & 0.05 & 0.8 & 0.15 \\
        \hline
        \multicolumn{4}{c}{} \\
        \multicolumn{4}{c}{$\pi_\textit{prv}(\textit{action}|\textit{tiger})$}
    \end{tabular}
    }
    \hspace*{\fill}
    \includegraphics[width=4cm]{toy2/noisy_good/js_nobs_512.pdf}
    \hspace*{\fill}
    \includegraphics[width=4cm]{toy2/noisy_good/reward_nobs_512.pdf}
    \hspace*{\fill}
    \caption{Noisy good agent.}
    \label{fig:toy2_noisy_good}
\end{figure}

\textbf{Random Agent}\\

In the random scenario there is no confounding, and observational data can be safely mixed with interventional data. The \emph{naive} thus does not suffer from any bias, and in fact is the one that converges the fastest to the optimal transition model and policy. Our method, while it manages to leverage the observational data to converge faster than \emph{no obs}, suffers from a worse performance than \emph{naive} in the low sample regime, most likely because it must try to recover a spurious confounding variable to distinguish the observational and interventional regimes.

\begin{figure}[ht]
    \hspace*{\fill}
    \raisebox{17mm}{
    % \footnotesize
    \scriptsize
    \begin{tabular}{|c|c|c|c|}
        \multicolumn{3}{c}{}\vspace{0.05 in}\\
        \cline{2-4}
        \multicolumn{1}{c}{} & \multicolumn{3}{|c|}{\textit{action}}     \\
        \hline
        \textit{tiger}  & listen & open left & open right \\
        \hline
        left       & 0.33 & 0.33 & 0.33 \\
        right      & 0.33 & 0.33 & 0.33 \\
        \hline
        \multicolumn{4}{c}{} \\
        \multicolumn{4}{c}{$\pi_\textit{prv}(\textit{action}|\textit{tiger})$}
    \end{tabular}
    }
    \hspace*{\fill}
    \includegraphics[width=4cm]{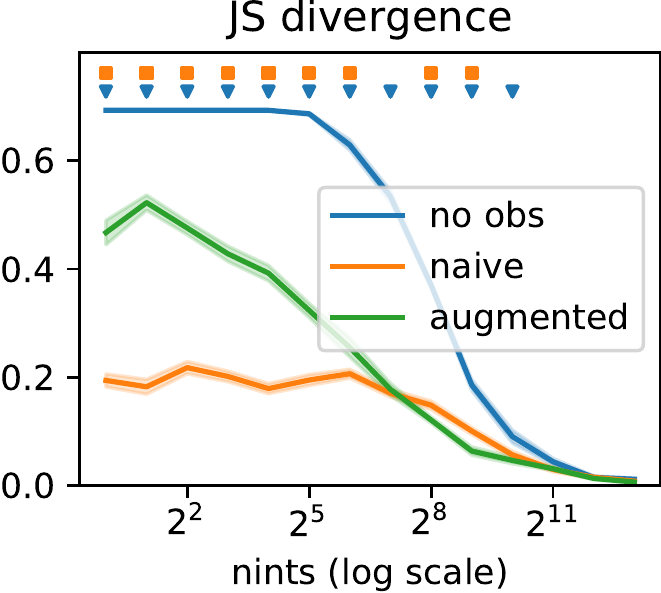}
    \hspace*{\fill}
    \includegraphics[width=4cm]{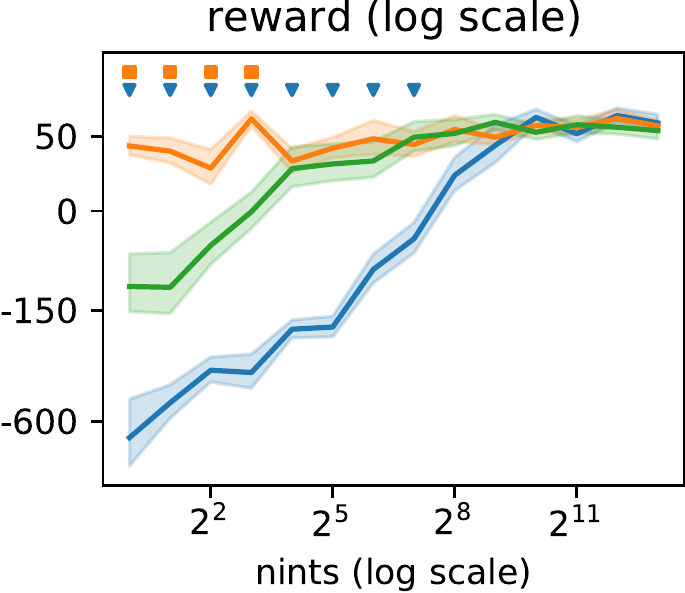}
    \hspace*{\fill}
    \caption{Random agent.}
    \label{fig:toy2_random}
\end{figure}

\clearpage
\textbf{Very Good Agent}\\

Here the privileged agent never opens the wrong door, and thus never receives the very penalizing -100 reward. As a result the \emph{naive} approach seems to be overly optimistic, which strongly affects the expected reward it obtains in the true environment. While our \emph{augmented} approach seems also to suffer from this bias in the very low sample regime (1, 2, 4 interventional trajectories), it is able to quickly overcome the bias and converges faster than \emph{no obs} to the optimal policy thanks to the observational data.

\begin{figure}[ht]
    \hspace*{\fill}
    \raisebox{17mm}{
    % \footnotesize
    \scriptsize
    \begin{tabular}{|c|c|c|c|}
        \multicolumn{3}{c}{}\vspace{0.05 in}\\
        \cline{2-4}
        \multicolumn{1}{c}{} & \multicolumn{3}{|c|}{\textit{action}}     \\
        \hline
        \textit{tiger}  & listen & open left & open right \\
        \hline
        left       & 0.05 & 0.0 & 0.95 \\
        right      & 0.05 & 0.95 & 0.0 \\
        \hline
        \multicolumn{4}{c}{} \\
        \multicolumn{4}{c}{$\pi_\textit{prv}(\textit{action}|\textit{tiger})$}
    \end{tabular}
    }
    \hspace*{\fill}
    \includegraphics[width=4cm]{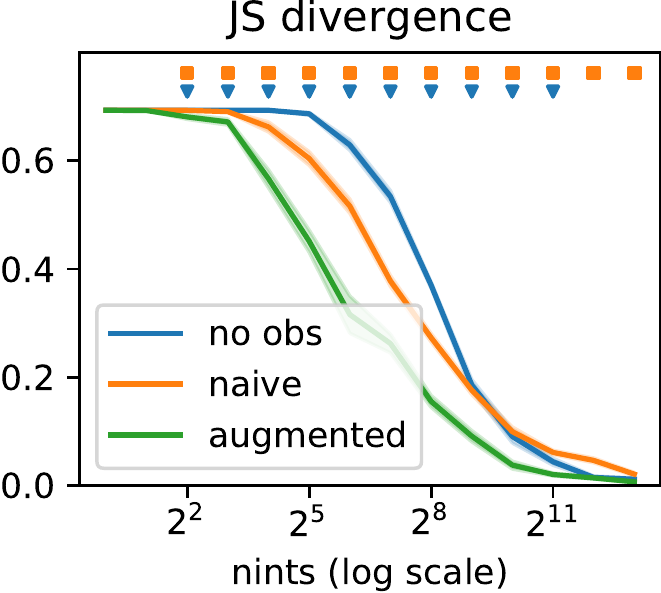}
    \hspace*{\fill}
    \includegraphics[width=4cm]{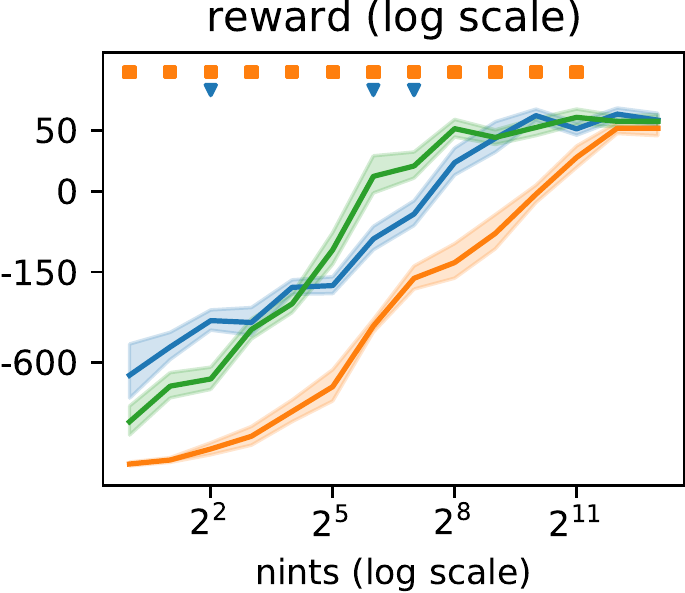}
    \hspace*{\fill}
    \caption{Very good agent.}
    \label{fig:toy2_very_good}
\end{figure}

\textbf{Very Bad Agent}\\

Here the privileged agent never opens the correct door, and thus never receives a positive reward (+10). As a result the \emph{naive} approach seems to be very conservative, and prefers not to take any chances opening a door in the low sample regime. It turns out that this strategy is not too bad in terms of reward (always listening yields a -51 total reward), and as such this observational bias seems to positively affect the performance of the \emph{naive} approach. Our \emph{augmented} method, on the other hand, seems to start taking more risks at the beginning, resulting in a worse reward performance despite a better JS divergence. Its performance eventually matches that of \emph{naive} in terms of reward, and it convergences to an optimal policy faster than \emph{no obs}.

\begin{figure}[ht]
    \hspace*{\fill}
    \raisebox{17mm}{
    % \footnotesize
    \scriptsize
    \begin{tabular}{|c|c|c|c|}
        \multicolumn{3}{c}{}\vspace{0.05 in}\\
        \cline{2-4}
        \multicolumn{1}{c}{} & \multicolumn{3}{|c|}{\textit{action}}     \\
        \hline
        \textit{tiger}  & listen & open left & open right \\
        \hline
        left       & 0.05 & 0.95 & 0.0 \\
        right      & 0.05 & 0.0 & 0.95 \\
        \hline
        \multicolumn{4}{c}{} \\
        \multicolumn{4}{c}{$\pi_\textit{prv}(\textit{action}|\textit{tiger})$}
    \end{tabular}
    }
    \hspace*{\fill}
    \includegraphics[width=4cm]{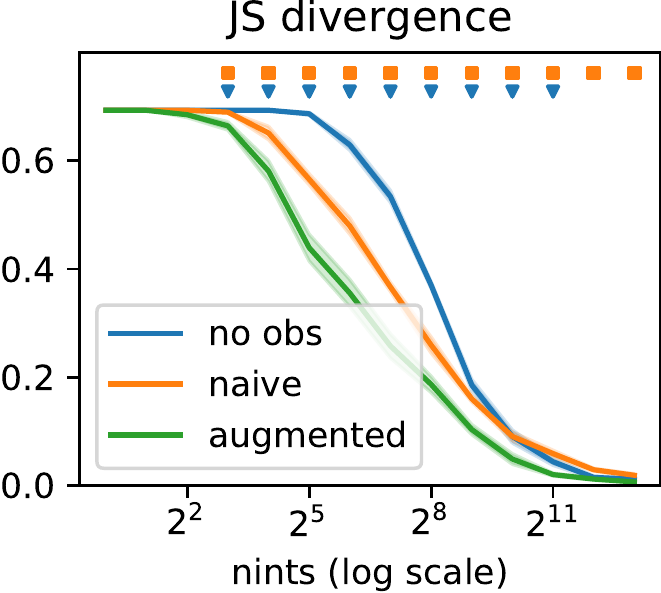}
    \hspace*{\fill}
    \includegraphics[width=4cm]{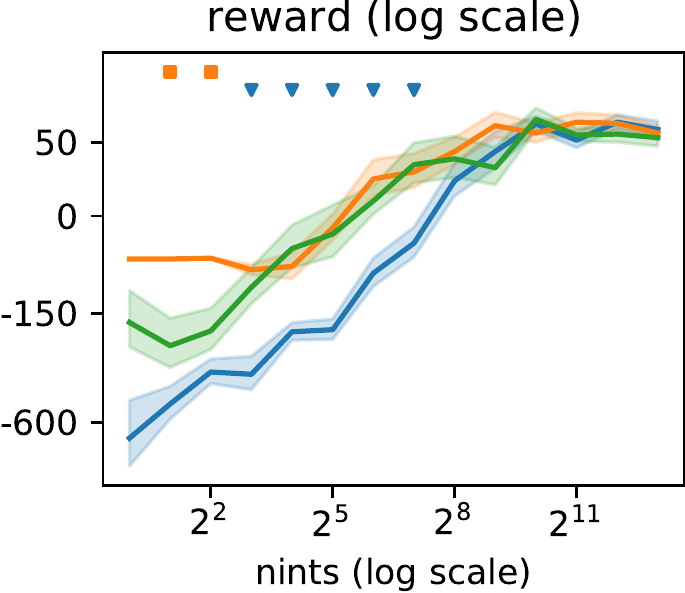}
    \hspace*{\fill}
    \caption{Very bad agent.}
    \label{fig:toy2_very_bad}
\end{figure}

\clearpage

\textbf{Optimistic, Right Door Biased Agent}\\

In this scenario, we induce a strong optimistic bias on both of the doors, similarly to the \emph{very good agent} scenario. In addition, we also infer a positive bias towards the right door, as the privileged agent decides to only open the wrong door when the tiger is behind the left door. The resulting behaviour for \emph{naive} and \emph{augmented} is very similar to what we see in the \emph{very good agent} scenario, with the \emph{naive} agent overestimating the potential reward behind each door, taking too much risks opening doors, and encountering a lot of tigers on the way.

\begin{figure}[ht]
    \hspace*{\fill}
    \raisebox{17mm}{
    % \footnotesize
    \scriptsize
    \begin{tabular}{|c|c|c|c|}
        \multicolumn{3}{c}{}\vspace{0.05 in}\\
        \cline{2-4}
        \multicolumn{1}{c}{} & \multicolumn{3}{|c|}{\textit{action}}     \\
        \hline
        \textit{tiger}  & listen & open left & open right \\
        \hline
        left       & 0.05 & 0.20 & 0.75 \\
        right      & 0.05 & 0.95 & 0.00 \\
        \hline
        \multicolumn{4}{c}{} \\
        \multicolumn{4}{c}{$\pi_\textit{prv}(\textit{action}|\textit{tiger})$}
    \end{tabular}
    }
    \hspace*{\fill}
    \includegraphics[width=4cm]{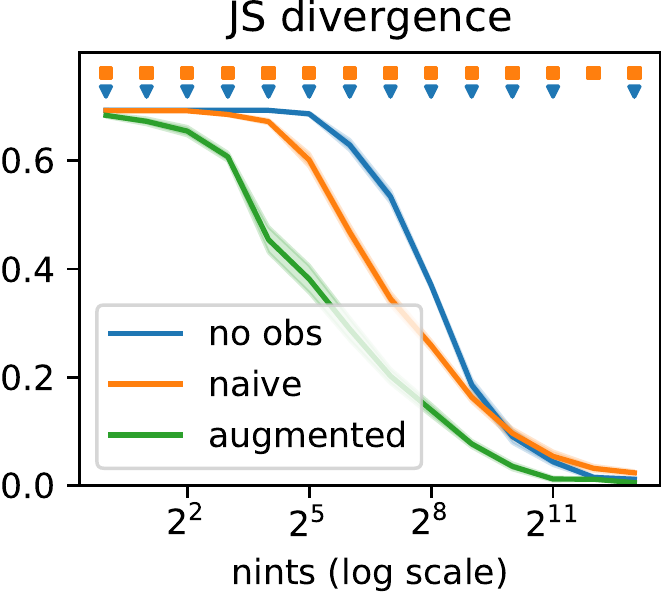}
    \hspace*{\fill}
    \includegraphics[width=4cm]{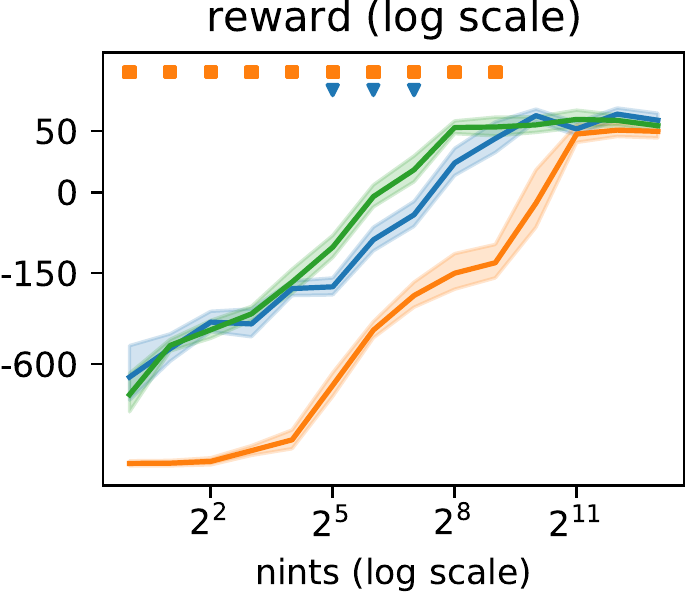}
    \hspace*{\fill}
    \caption{Optimistic, right door biased agent.}
    \label{fig:toy2_strong_good_bias}
\end{figure}

\textbf{Pessimistic, Right Door Biased Agent}\\

On the opposite, here we induce a strong pessimistic bias on both of the doors, similarly to the \emph{very bad agent} scenario. In addition, we again infer a positive bias towards the right door, as the privileged agent decides to only open the correct door when the tiger is behind the right door. The resulting behaviour for \emph{naive} and \emph{augmented} is again very similar to what we see in the \emph{very bad agent} scenario.

\begin{figure}[ht]
    \hspace*{\fill}
    \raisebox{17mm}{
    % \footnotesize
    \scriptsize
    \begin{tabular}{|c|c|c|c|}
        \multicolumn{3}{c}{}\vspace{0.05 in}\\
        \cline{2-4}
        \multicolumn{1}{c}{} & \multicolumn{3}{|c|}{\textit{action}}     \\
        \hline
        \textit{tiger}  & listen & open left & open right \\
        \hline
        left       & 0.05 & 0.95 & 0.0 \\
        right      & 0.05 & 0.20 & 0.75 \\
        \hline
        \multicolumn{4}{c}{} \\
        \multicolumn{4}{c}{$\pi_\textit{prv}(\textit{action}|\textit{tiger})$}
    \end{tabular}
    }
    \hspace*{\fill}
    \includegraphics[width=4cm]{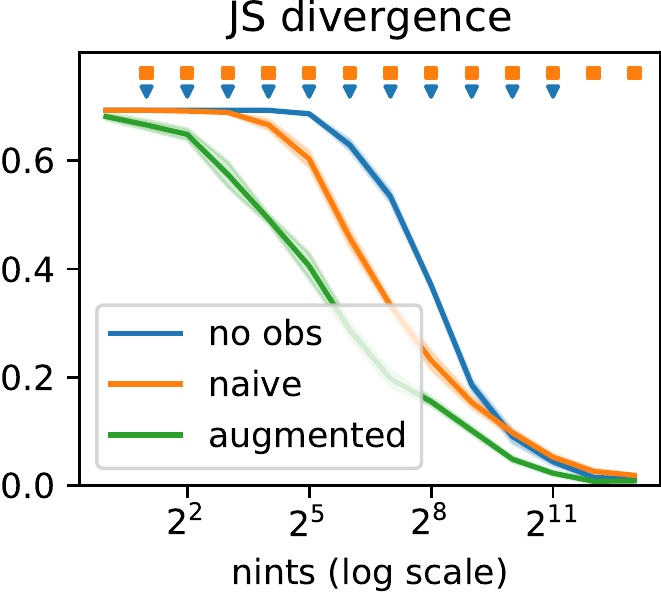}
    \hspace*{\fill}
    \includegraphics[width=4cm]{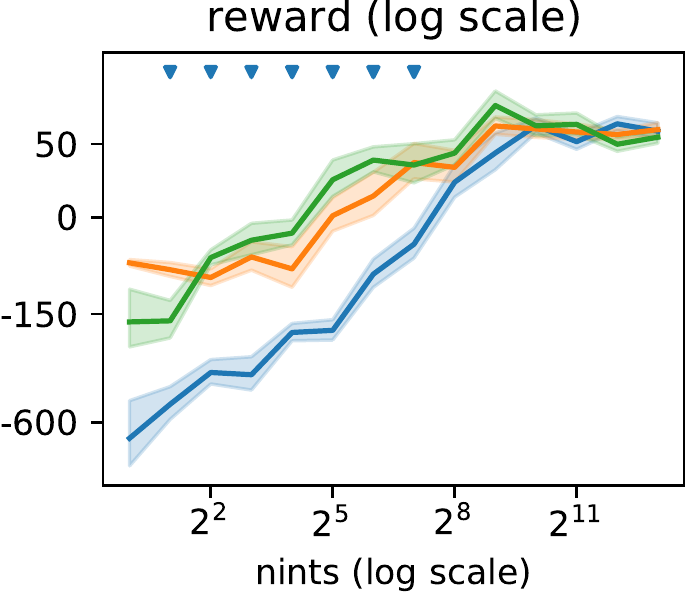}
    \hspace*{\fill}
    \caption{Pessimistic, right door biased agent.}
    \label{fig:toy2_strong_bad_bias}
\end{figure}

\end{document}